\documentclass{article}

\usepackage[preprint, nonatbib]{neurips_2022}
\usepackage{verbatim}
\usepackage[utf8]{inputenc}
\usepackage{bm,subfigure}
\usepackage{epsfig,amstext,xspace}
\usepackage[ruled,linesnumbered]{algorithm2e}% 

\usepackage[utf8]{inputenc} % allow utf-8 input
\usepackage[T1]{fontenc}    % use 8-bit T1 fonts
\usepackage{url}            % simple URL typesetting
\usepackage{booktabs}       % professional-quality tables
\usepackage{amsfonts}       % blackboard math symbols
\usepackage{bbm}
\usepackage{graphicx}
\usepackage{bbm}

\usepackage{inconsolata}
\usepackage{geometry}
\usepackage[round]{natbib}
\usepackage{amsthm}
\usepackage{microtype}
\usepackage{times}

\usepackage{upgreek}

\usepackage{mathtools}
\usepackage{thm-restate}
\usepackage{hyperref}
\usepackage{cleveref}

\crefname{lemma}{lemma}{lemmas}
\Crefname{lemma}{Lemma}{Lemmas}

\usepackage{amsmath}
\usepackage{amsfonts}
\usepackage{authblk}

\theoremstyle{plain}
\newtheorem{theorem}{Theorem}[section]

    \newtheorem{lemma}[theorem]{Lemma}
    
    \newtheorem{claim}[theorem]{Claim}

    \newtheorem{corollary}[theorem]{Corollary}

\theoremstyle{definition}
 
    \newtheorem{remark}[theorem]{Remark}

\theoremstyle{plain}

\usepackage{tikz}
\usetikzlibrary{arrows}

\usepackage{nicefrac}

\newcommand{\xhdr}[1]{\vspace{1mm} \noindent{\bf #1}}

\newcommand{\ledger}{\uplambda}

\newcommand{\calC}{\mathcal{C}}
\newcommand{\calA}{\mathcal{A}}

\newcommand{\calF}{\mathcal{F}}
\newcommand{\calD}{\mathcal{D}}
\newcommand{\calP}{\mathcal{P}}

\newcommand{\calK}{\mathcal{K}}
\newcommand{\roundhall}{k^{\mathrm{hal}}}

\newcommand{\ttbf}[1]{{\normalfont \texttt{\textbf{#1}}}}

\newcommand{\opfont}[1]{\ttbf{#1}}
\newcommand{\ledhonl}[1][\ell]{\ledger_{\mathrm{hon};#1}}
\newcommand{\ledhall}[1][\ell]{\ledger_{\mathrm{hal};#1}}

\newcommand{\ledrawl}[1][\ell]{\ledger_{\mathrm{raw};#1}}
\newcommand{\ledcensl}[1][\ell]{\ledger_{\mathrm{cens};#1}}
\newcommand{\ledul}[1][\ell]{\ledger_{\mathcal{U}_{\ell};\ell}}

\newcommand{\algcomment}[1]{\textcolor{blue!70!black}{\footnotesize{\texttt{\textbf{#1}}}}}

\newcommand{\abs}[1]{\left\lvert#1\right\rvert}

\newcommand{\nrm}[1]{\left\lVert#1\right\rVert}

\let\Pr\undefined

\newcommand{\N}{\mathbb{N}}
\newcommand{\R}{\mathbb{R}}

\newcommand{\E}{\mathbb{E}}

\DeclareMathOperator{\Pr}{\mathbb{P}}

\DeclareMathOperator*{\argmax}{arg\,max}
\DeclareUnicodeCharacter{2212}{-}

\newcommand{\Hallucinate}{\textsc{Hallucinate}}
\newcommand{\kexpl}{k_{\ell}}
\newcommand{\kexpi}{k_{i}}

\newcommand{\phase}{\opfont{phase}}
\newcommand{\nphase}{n_{\mathrm{ph}}}
\newcommand{\nlearn}{n_{\mathrm{lrn}}}
\newcommand{\epspun}{\epsilon_{\mathrm{pun}}}

\newcommand{\vR}{\mathbf{R}}
\newcommand{\eye}{\mathbf{I}}

\newcommand{\EvPunl}{\mathcal{E}_{\mathrm{pun}, \ell}}

\newcommand{\ralt}{\mathrm{r}_{\mathrm{alt}}}
\newcommand{\qpun}{\mathrm{q}_{\mathrm{pun}}}
\newcommand{\thetahall}{\theta_{\mathrm{hal},\ell}}
\makeatletter
\def\namedlabel#1#2{\begingroup
   \def\@currentlabel{#2}%
   \label{#1}\endgroup
}
\makeatother
\newcommand{\supt}{^{(t)}}

\newcommand{\myNp}{{n_\calP}}

\newcommand{\myEv}{\mathcal{E}}
\newcommand{\feaSet}{\term{FEASIBLE}}
\newcommand{\goodSet}{\term{GOOD}}
\newcommand{\badSet}{\term{BAD}}

\usepackage{nicefrac}

\renewcommand{\xhdr}[1]{\vspace{0mm} \noindent{\bf #1}}
\newcommand{\OMIT}[1]{}

\newcommand{\term}[1]{\ensuremath{\mathtt{#1}}\xspace}
\newcommand{\TS}{\term{TS}}
\newcommand{\ALG}{\term{ALG}}
\newcommand{\algHH}{\term{HiddenHallucination}}

\newcommand{\ie}{{\em i.e.,~\xspace}}

\newcommand{\eg}{{\em e.g.,~\xspace}}

\newcommand{\rbr}[1]{\left(\,#1\,\right)}
\newcommand{\sbr}[1]{\left[\,#1\,\right]}
\newcommand{\cbr}[1]{\left\{\,#1\,\right\}}

\newcommand{\LDOTS}{\, ,\ \ldots\ ,}     % smart "..."
\newcommand{\eps}{\epsilon}
\newcommand{\poly}{\mathrm{poly}}
\newcommand{\Cel}[1]{{\lceil {#1} \rceil}}

\newcommand{\initOneLiners}{%
         \setlength{\itemsep}{0pt}
        \setlength{\parsep}{0pt}
          \setlength{\topsep }{0pt}
}

\usepackage{xcolor}         % colors
\usepackage[suppress]{color-edits}

\addauthor{as}{red}
\addauthor{dn}{magenta}
\addauthor{vh}{blue}
\addauthor{sw}{purple}

\title{Incentivizing Combinatorial Bandit Exploration}

\begin{document}
\author[1]{Xinyan Hu}
\author[2]{Dung Daniel Ngo}
\author[3]{Aleksandrs Slivkins}
\author[4]{Zhiwei Steven Wu}
\affil[1]{Peking University, xy.hu@pku.edu.cn}
\affil[2]{University of Minnesota, ngo00054@umn.edu}
\affil[3]{Microsoft Research, slivkins@microsoft.com}
\affil[4]{Carnegie Mellon University, zstevenwu@cmu.edu}
\maketitle

%need to think about titles. running candidates:
%\begin{itemize}
%\item Incentivizing Exploration for Combinatorial (Semi)-Bandits
%\item Bayesian Exploration for Combinatorial Action Space
%\item Incentivizing Bandit Exploration with  Hallucinated Thompson.... ok, a bit crazy, but meomorable
%\item \vhcomment{Incentiving Combinatorial Bandit Exploration with Cyclic-based Thompson}
%\end{itemize}

\begin{abstract}
Consider a bandit algorithm that recommends actions to self-interested users in a recommendation system. The users are free to choose other actions and need to be incentivized to follow the algorithm's recommendations. While the users prefer to \emph{exploit}, the algorithm can incentivize them to \emph{explore} by leveraging the information collected from the previous users. All published work on this problem, known as \emph{incentivized exploration}, focuses on small, unstructured action sets and mainly targets the case when the users' beliefs are independent across actions. However, realistic exploration problems often feature large, structured action sets and highly correlated beliefs. We focus on a paradigmatic exploration problem with structure: combinatorial semi-bandits. We prove that Thompson Sampling, when applied to combinatorial semi-bandits, is incentive-compatible when initialized with a sufficient number of samples of each arm (where this number is determined in advance by the Bayesian prior). Moreover, we design incentive-compatible algorithms for collecting the initial samples.

\end{abstract}

\section{Introduction}
\label{sec:intro}
We consider \emph{incentivized exploration}: how to incentivize self-interested users to explore. A social planner interacts with self-interested users (henceforth, \emph{agents}) and can make recommendations, but cannot enforce the agents to comply with these recommendations. The agents face uncertainty about the available alternatives. The social planner would like the agents to trade off \emph{exploration} for the sake of acquiring new information and \emph{exploitation}, making optimal near-term decisions based on the current information. The agents, on the other hand, prefer to \emph{exploit}. However, the algorithm can incentivize them to \emph{explore} by leveraging the information collected from the previous users. This problem has been studied since \citet{Kremer-JPE14}, see \citet[Ch. 11]{slivkins-MABbook} for an overview.

All published work on this problem focuses on small, unstructured action sets. Moreover, the case of \emph{independent priors} -- when the users' beliefs are independent across actions -- is emphasized as the main, paradigmatic special case when specific performance guarantees are derived. However, realistic exploration problems often feature large sets with some known structure that connects actions to one another. A major recurring theme in the vast literature on multi-armed bandits is taking advantage of the available structure so as to enable the algorithm to cope with the large number of actions.

%\citep{Bubeck-survey12,LS19bandit-book,slivkins-MABbook}

We focus on a paradigmatic, well-studied exploration problem with structured actions: \emph{combinatorial semi-bandits}. Here, each arm is a subset of some ground set, whose elements are called \emph{atoms}. In each round, the algorithm chooses an arm, and observes/collects reward for each atom in this arm. The reward for each atom is drawn independently from some fixed (but unknown) distribution specific to this atom. The set of feasible arms reflects the structure of the problem, \eg it can comprise all subsets of  atoms of a given cardinality, or all edge-paths in a given graph. Since the number of arms ($K$) can be exponential in the number of atoms ($d$), the main theme is replacing the dependence on $K$ in regret bounds for ``unstructured" $K$-armed bandits with a similar dependence on $d$.

We adopt a standard model for incentivized exploration from \citet{Kremer-JPE14}. The social planner is implemented as a bandit algorithm. Each round corresponds to a new agent which arrives and receives the arm chosen by the algorithm as a recommendation. Agents have Bayesian beliefs, independent across the atoms (but highly correlated across the arms). The algorithm must ensure that following its recommendation is in each agent's best interest, a condition called \emph{Bayesian incentive-compatibility} (\emph{BIC}). Each agent does not observe what happened with the previous agents, but the algorithm does. This information asymmetry is crucial for creating incentives.

\xhdr{Our contributions.}
We prove that Thompson Sampling is BIC when initialized with at least $n_{\TS}$ samples of each atom, where $n_{\TS}$ is determined by the prior and scales polynomially in the number of atoms ($d$). Thompson Sampling \citep{Thompson-1933} is a well-known bandit algorithm with near-optimal regret bounds and good empirical performance. The initial samples can be provided by another BIC algorithm (more on this below), or procured exogenously, \eg bought with money.

Next, we consider the problem of \emph{initial exploration}: essentially, design a BIC algorithm that samples each atom at least once. Such algorithms are interesting in their own right, and can be used to bootstrap Thompson Sampling, as per above. We present two such algorithms, which build on prior work \citep{ICexploration-ec15,IncentivizedRL} and extend it in non-trivial ways. The objective to be optimized is the sufficient number of rounds $T_0$, and particularly its dependence on $d$. To calibrate, prior work on incentivized exploration in multi-armed bandits with correlated priors does not provide any guarantees for a super-constant number of arms ($K$), and is known to have $T_0 > \exp(\Omega(K))$ in some natural examples \citep{ICexploration-ec15}. In contrast, our algorithms satisfy $T_0 \leq \exp(O(d))$ for a paradigmatic special case, and $T_0 \leq \exp(O(d^3))$ in general.

Finally, what if the prior is \emph{not} independent across atoms? We focus on two arms with arbitrary correlation, a fundamental special case of incentivized exploration, and prove that our analysis of Thompson Sampling extends to the case. This result may be of independent interest.

\xhdr{Discussion.}
Like all prior work on incentivized exploration, we consider standard, yet idealized models for agents' economic behavior and the machine-learning problem being solved by the social planner. The modeling captures something essential about exploration and incentives in recommendation systems, but is not supposed to capture all the particularities of any specific application scenario. The goal of this paper is to bring more complexity into the machine-learning problem; advancing the economic model is beyond our scope.

We focus on establishing the BIC property and asymptotic guarantees in terms of the number of atoms, without attempting to optimize the dependence on the per-atom Bayesian priors. Our Thompson Sampling result has an encouraging practical implication: a standard, well-performing bandit algorithm plays well with users' incentives, provided a small (in theory) amount of initial data.

The significance of focusing on combinatorial semi-bandits is primarily that it is a paradigmatic example of exploration with large, structured action sets, with a number of motivating examples established in the literature. (Stylized) motivating examples specific to incentivized exploration include: recommending online content, \eg for news or entertainment (with atoms as \eg specific news articles); recommending complementary products, \eg a suit that consists of multiple items of clothing; recommending driving directions. In all cases, the social planner corresponds to the online platform issuing the respective recommendations. Such online platforms are often interested in maximizing users' happiness, rather than (or in addition to) the immediate revenue, as a way to ensure user engagement and long-term success.

As a theoretical investigation, this paper appears unlikely to cause social harms. If anything, our goal here is to ensure that the aggregate welfare is not harmed by users' myopia.

\xhdr{Related work.}
Incentivized exploration, as defined in this paper, has been introduced in \citet{Kremer-JPE14} and subsequently studied, \eg in
\citet{ICexploration-ec15,ICexplorationGames-ec16,Jieming-unbiased18,Bahar-ec16,Bahar-ec19}, along with some extensions. Most related is \citet{Selke-PoIE-ec21}, which obtains similar BIC results for the special case of multi-armed bandits with independent priors, both for Thompson Sampling and for initial exploration. A yet unpublished working paper of \cite{IncentivizedRL} provides a BIC algorithm for initial exploration in reinforcement learning; we build on this result in one of ours. Similar, but technically incomparable versions have been studied, \eg with time-discounted rewards \citep{Bimpikis-exploration-ms17} and creating incentives via money
\citep{Frazier-ec14,Kempe-colt18}.

From the perspective of theoretical economics, incentivized exploration is related to the literature on information design \citep{Kamenica-survey19,BergemannMorris-survey19}: essentially, one round of incentivized exploration is an instance of Bayesian persuasion, a central model in this literature. Other ``online" models of Bayesian persuasion have been studied \citep[\eg][]{DBLP:conf/nips/CastiglioniCM020,DBLP:conf/sigecom/ZuIX21}, but are very different from ours in that the planner's problem has nothing to do with exploration, and is not even meaningful without incentives.

%Our work draws on techniques from the growing literature of incentivized exploration (IE) \citep{kremer2014implementing, mansour2015bayesian, IncentivizedRL} where the goal is to incentivize myopic agents to explore arms in a multi-armed bandit setting using information asymmetry techniques from Bayesian persuasion \citep{kamenica2011persuasion}. While our mechanisms are technically related to those in \citet{mansour2015bayesian} and \citet{Selke-PoIE-ec21}, our work differs in several key aspects. First, prior work in incentivizing exploration using Thompson sampling  \citep{Selke-PoIE-ec21} does not capture the setting where the actions are correlated. This absence of independence in actions and reward deviates from the traditional multi-arm bandit framework, which is captured by our setting of combinatorial bandit. Second, the Hidden Exploration technique used in \citet{mansour2015bayesian} can be applied to our setting as they allow for correlated priors, but their regret rate depends on total number of arms, which is exponential in the combinatorial bandit setting. By contrast, our mechanism relies only on exploring a smaller number of arms that form a basis of the arm space to kick start Thompson Sampling.

On the machine learning side, this paper is related to the work on combinatorial semi-bandits, starting from \citet{Gyorgy-jmlr07}, \eg  \citep{Chen-icml13,Kveton-aistats15,MatroidBandits-uai14}, and the work on Thompson Sampling, see \citet{TS-survey-FTML18} for a survey. In particular, near-optimal Bayesian regret bounds have been derived in \citet{Russo-MathOR-14,Russo-JMLR-16}, and frequentist ones in 
\citep{Shipra-aistats13-JACM,Kaufmann-alt12}. Thompson Sampling has been applied to combinatorial semi-bandits,  \citep[\eg][]{gopalan2014ThompsonSF, Wen-icml15, degenne2016combinatorial, wang2018tscb}, with Bayesian regret bounds derived in \citet{Russo-JMLR-16}.

%There has been a line of work orthogonal to ours on mechanisms that incentivize exploration via payment \citep{frazier2014incentivizing, kannan2017fairness, chen2018incentivizing}. There are several known disadvantages of such payment mechanisms, including potential high costs and ethical concerns \citep{groth2010honorarium}. For a detailed discussion, see \citet{slivkins2017incentivizing}. 
\section{Problem Formulation and Preliminaries}
\label{sec:prelim}

Our algorithm operates according to the standard protocol for \emph{combinatorial semi-bandits}, with an ancillary incentive-compatibility constraint, standard in the literature on \emph{incentivized exploration}.

\xhdr{Combinatorial semi-bandits.}
There are $T$ rounds, $d$ atoms and $K$ arms, where each arm is a subset of atoms. The set $\calA$ of feasible arms is fixed and known. In each round $t$, each atom $\ell$ generates reward $r\supt_\ell\in[0,1]$. The algorithm chooses an arm $A\supt\in\calA$ and observes the reward of each atom in this arm (and nothing else). Algorithm's reward in this round is the total reward of these atoms.

%A typical goal is to maximize \emph{regret}: the difference in the expected total reward (over all rounds) between the algorithm and the best fixed action.

%The feasible set $\calA$ reflects the structure of the problem, \eg a collection of all subsets of atoms of a given size, or the set of all edge-paths in a given graph. \asedit{The former is a paradigmatic special case, where atoms represent pieces of content that can be presented to a user.} Since $K = |\calA|$ can be exponential in $d$, the main theme when going from $K$-armed bandits to combinatorial semi-bandits is replacing the dependence on $K$ in regret bounds with a similar dependence on $d$.

Formally, we write
    $[T] := \{1 \LDOTS T\}$
for the set of all rounds and $[d]$ for the set of all atoms, so that arms are subsets $A\subset [d]$. Let $\theta_\ell$ be the expected reward of atom $\ell\in[d]$, and let
    $\mu(A) = \sum_{\ell\in A} \theta_\ell$ be the expected reward of a given arm $A\subset [d]$.
Note that $d$-armed bandits are a special case when the feasible arms are singleton sets $\{\ell\}$, $\ell\in[d]$.

\xhdr{Stochastic rewards and Bayesian priors.}
The reward of each atom $\ell\in [0,1]$ is drawn independently in each round from a fixed distribution $\calD_\ell$ specific to this atom. This distribution comes from a parametric family, parameterized by the expected reward $\theta_\ell$. The (realized) problem instance is therefore specified by
the \emph{mean reward vector}
    $\mathbf{\theta} = (\theta_1 \LDOTS \theta_d)$.
Initially, each $\theta_\ell$ is drawn independently from a Bayesian prior $\calP_\ell$ with support $\Theta\subset[0,1]$. Put differently, the mean reward vector $\mathbf{\theta}$ is drawn from the product prior
$\calP=\calP_{1}\times\cdots\times\calP_{d}$.
% This condition holds without loss of generality.
%Throughout our analysis, we refer to different variants of this necessary condition.

\xhdr{Incentive-compatibility.}
The algorithm must ensure that in each round $t$, conditional on a particular arm $A=A\supt$ being chosen, the expected reward of this arm is at least as good as that of any other arm. Formally,  the algorithm is called \emph{Bayesian incentive-compatible} (BIC) if for each round $t\in[T]$,
\begin{align}\label{def:BIC}
    \E[\; \mu(A) - \mu(A') \mid A\supt = A\;] \geq 0
    \qquad\forall \text{ arms } A,A'\in\calA \text{ with } \Pr[A\supt=A]>0.
\end{align}

This definition is based on the following stylized story . In each round $t$, a new user arrives to a recommendation system, observes the arm $A\supt$ chosen by our algorithm, and interprets it as a recommendation. Then the user decides which arm to choose (not necessarily the arm recommended), and receives the corresponding reward. Accordingly, the user needs to be incentivized to follow the recommendation. We adopt a standard setup from economic theory (and the prior work on incentivized exploration): each user has the same prior $\calP$, knows the algorithm, and wishes to maximize her expected reward. We posit that the user does not observe anything else before making her decision, other than the recommended arm. In particular, she does not observe anything about the previous rounds. Then, \eqref{def:BIC} ensures that she is (weakly) incentivized to follow her recommendation, assuming that the previous users followed theirs. We posit that under \eqref{def:BIC}, the user does follow recommendations, and then reports the rewards of all atoms to the algorithm.

We emphasize that this story is not a part of our formal model (although it can be expressed as such if needed). In fact, the story can be extended to allow the algorithm to reveal an arbitrary ``message" to each user, but this additional power is useless: essentially, anything that can be achieved with arbitrary messages can also be achieved with arm recommendations. This can easily be proved as a version of Myerson's \emph{direct revelation principle} from theoretical economics.

%In this paper, we use \emph{arms} and \emph{actions} interchangeably. \vhdelete{We index arms by $1\le i\le K$. We refer to them as “arm $i$” or “arm $A_i$” interchangeably. }\vhedit{We index arms by $1\le i\le K$ and refer to them as "arm $A_i$". We index atoms by $1\le \ell \le d$ and refer to them as "atom $\ell$" and "atom $a_\ell$". Each $A_i$ is a $d$-dimension vector. Arm $A_i$ includes atom $\ell$, if and only if the $\ell$-th coordinate of $A_i$ is $1$.}

\xhdr{Conventions.}
Each atom $\ell$ satisfies $\Pr[\theta_\ell>0]>0$: else, its rewards are all $0$, so it can be ignored.

No arm is contained in any the other arm. This is w.l.o.g. for Bernoulli rewards, and more generally if
    $\Pr[r\supt_\ell=0\mid \theta_\ell>0]>0$:
if $A\subset A'$ for arms $A,A'$ then $A$ cannot be chosen by any BIC algorithm.

W.l.o.g., order the atoms $\ell$ by their prior mean rewards $\theta_\ell^0 := \E[\theta_\ell]$, so that $\E[\theta_1^0] \geq \cdots \geq \E[\theta_d^0]$.

Let $A^* = \argmax_{A\in \calA} \mu(A)$ denote the best arm overall, with some fixed tie-breaking rule.

By a slight abuse of notation, each arm $A\subset [d]$ is sometimes identified with a binary vector $v\in \{0,1\}^d$ such that
    $v_\ell=1 \Leftrightarrow \ell\in A$,
for each atom $\ell\in[d]$. In particular, we write $A_\ell = v_\ell$.

\xhdr{Thompson Sampling} has a very simple definition, generic to many versions of multi-armed bandits. Let $\calF_t$ denote the realized history (tuples of chosen actions and realized rewards of all atoms) up to and not including round $t$. Write
    $\E\supt\sbr{\cdot} = \E\sbr{\cdot \mid \calF_t}$ and $\Pr\supt\sbr{\cdot} = \Pr\sbr{\cdot \mid \calF_t}$
as a shorthand for posterior updates. Thompson Sampling in a given round $t$ draws an arm independently at random from distribution $p\supt(A) = \Pr\supt[A^* = A]$, $A\in\calA$. If Thompson Sampling is started from some fixed round $t_0>1$, this is tantamount to starting the algorithm from round $1$, but with prior
    $\calP(\cdot\mid \calF_{t_0})$
rather than $\calP$. The algorithm is well-defined for an arbitrary prior $\calP$.

While this paper is not concerned with computational issues, they are as follows.
The posterior update $\calP(\cdot\mid \calF_t)$ can be performed for each atom $\ell$ separately:
$\calP_\ell(\cdot\mid \calF_{t,\ell})$, where $\calF_{t,\ell}$ is the corresponding history of samples from this atom. A standard implementation draws $\theta'_\ell\in[0,1]$ independently from $\calP_\ell(\cdot\mid \calF_{t,\ell})$, for each atom $\ell$, then chooses the best arm according to these draws:
    $\argmax_{A\in\calA} \sum_{\ell\in A} \theta'_\ell$.
The posterior updates $\calP_\ell(\cdot\mid \calF_{t,\ell})$ and the $\argmax$ choice are not computationally efficient in general, and may require heuristics (this is a common situation for all variants of Thompson Sampling). A paradigmatic special case is Beta priors $\calP_\ell$ and Bernoulli reward distributions $\calD_\ell$, so that the posterior update $\calP_\ell(\cdot\mid \calF_{t,\ell})$ is another Beta prior.

\xhdr{Composition of BIC algorithms.}
We rely on a generic observation from \citet{ICexploration-ec15,ICexplorationGames-ec16} that the composition of two BIC algorithms is also BIC.

\begin{lemma}\label{lm:prelims-composition}
Let \ALG be a BIC algorithm which stops after some round $T_0$. Let $\ALG'(H)$ be another algorithm that initially inputs the history $H$ collected by \ALG, and suppose it is BIC. Consider the composite algorithm: \ALG followed by $\ALG'(H)$, which stops at the time horizon $T$. If $T_0$ is determined before the composite algorithm starts, then this algorithm is BIC.
\end{lemma}

\section{Thompson Sampling is BIC}
\label{sec:ts}
Our main result is that Thompson Sampling is BIC when initialized with at least $n_{\TS}$ samples of each atom, where $n_{\TS}$ is determined by the prior and scales polynomially in $d$, the number of atoms.

\begin{theorem}
\label{thm:semi-ts-bic}
Let \ALG be any BIC algorithm such that by some time $T_0$ (which is determined by the prior) it almost surely collects at least $n_{\TS} = C_{\TS} \cdot d^2 \cdot  \eps^{-2}_{\TS} \cdot  \log(\delta^{-1}_{\TS})$ samples from each atom, where
\begin{align}
    \eps_{\TS} = \min_{A, A' \in \calA} \E[(\mu(A) - \mu(A'))_{+}] \quad \text{and} \quad \delta_{\TS} = \min_{A \in \calA} \Pr[A^* = A],
\end{align}
and $C_{\TS}$ is a large enough absolute constant. Consider the composite algorithm which runs \ALG for the first $T_0$ rounds, followed by Thompson sampling. This algorithm is BIC.
\end{theorem}

Note that $T_0$ and $n_{\TS}$ are ``constants" once the prior is fixed, in the sense that they do not depend on the time horizon $T$, the mean reward vector $\theta$, or the rewards in the data collected by \ALG.

\begin{remark}
For statistical guarantees, consider Bayesian regret, \ie regret in expectation over the Bayesian prior.
Bayesian regret of the composite algorithm in Theorem~\ref{thm:semi-ts-bic} is at most $T_0$ plus Bayesian regret of Thompson Sampling. The latter is $O(\sqrt{d T\log d})$ for any prior \citep{Russo-JMLR-16}. For an end-to-end result, we provide a suitable \ALG in \Cref{subsec:cyclic}, with a specific $T_0$.
\end{remark}

\begin{remark}
We can invoke Lemma~\ref{lm:prelims-composition} since $T_0$ is determined in advance. So, it suffices to show that each round $t$ of Thompson Sampling satisfies the BIC condition \eqref{def:BIC}.
\end{remark}

Let us clarify the dependence on $d$. Note that parameters $\eps_{\TS}$ and $\delta_{\TS}$ may depend on $d$ through the prior. To separate the dependence on $d$ from that on the prior, we posit that each per-atom prior $\calP_\ell$ belongs to a fixed collection $\calC$. We make mild non-degeneracy assumptions:%
\footnote{For the special case of $d$-armed bandits, assumption \eqref{eq:TS-pairwise} is necessary and sufficient for the respective arm $A$ to be \emph{explorable}: chosen in some round by some BIC algorithm \citep{Selke-PoIE-ec21}.}
\begin{align}
\Pr\sbr{ \mu(A') < \E[\mu(A)]} > 0
    &\qquad\text{for all arms $A\neq A'$}.
    \label{eq:TS-pairwise}\\
\Pr[ \theta_\ell>\tau ] > 0
    &\qquad\text{for all atoms $\ell\in[d]$ and some $\tau\in(0,1)$}.
    \label{eq:TS-full-support}\\
\Pr[ \theta_\ell< x ] > \poly(\nicefrac{1}{x})\cdot \exp(-x^{-\alpha})
    &\qquad\text{for all atoms $\ell\in[d]$, $x\in(0,\nicefrac12)$ and some $\alpha\geq 0$}.
    \label{eq:TS-nondegeneracy}
\end{align}

\begin{corollary} Suppose all priors $\calP_\ell$ of atoms $\ell \in [d]$ belong to some fixed, finite collection $\calC$ of priors and assumptions (\ref{eq:TS-pairwise}-\ref{eq:TS-nondegeneracy}) are satisfied with some absolute constants $\alpha,\tau$. Then $n_{\TS} = O_{\calC}(d^{3+\alpha} \log d)$, where $O_\calC$ hides the absolute constants {and the dependency on $\calC$.}
%
%Moreover, we obtain $n_{\TS} = O_{\calC}(d^3)$ in the special case where all arms in $\calA$ have the same size
\label{cor:ts-sufficient-samples}
\end{corollary}

\begin{remark}
The initial data can also be provided to Thompson Sampling exogenously (rather than via a BIC algorithm \ALG), \eg purchased with money. More formally, one would need to provide a collection of (arm, reward) datapoints such that each atom is sampled at least $n_{\TS}$ times.%
\footnote{A subtlety: the number of samples of each arm should be known in advance. This is because otherwise Bayesian update on this data may become dependent on the data-collection algorithm.}
\end{remark}

%In the remainder of this section we sketch the proofs for \Cref{thm:semi-ts-bic} and \Cref{cor:ts-sufficient-samples}.

% \ascomment{stopped here}

\begin{proof}[Proof Sketch for \Cref{thm:semi-ts-bic}  (full proof in~\Cref{appendix:semi-ts-bic-appendix})]

In order to establish the BIC condition in \eqref{def:BIC} for Thompson Sampling, we first observe that $\Pr[A^*=A]$ is a positive prior-dependent constant for all arms $A$, so it suffices to prove $\E \left[\E\supt[\mu(A) - \mu(A')] \cdot \mathbf{1}_{\{A^* = A \}} \right] \geq 0$ for all $A, A'$.

\OMIT{
\swdelete{The main technical tool used in the analysis of \Cref{thm:semi-ts-bic} is the Harris inequality (\Cref{thm:harris-inequality}), a correlation inequality that says increasing functions of independent random variables are non-negatively correlated. \dndelete{Because of the independent assumption across dimensions on the mean reward vector $\theta$, we can use the Harris inequality in our analysis to bound the posterior gap between the recommended arm and any other arm to ensure the BIC property is satisfied.}
Our main technical contribution here is how we leverage the independent assumption across atoms and the linear structure of combinatorial semi-bandits in order to use Harris inequality and bound the absolute estimation error in estimating the posterior gap between arms.

First, using the definition of Thompson sampling and $\Pr[A^*=A]$ is a positive prior-dependent constant, it suffices to prove $\E \left[\E\supt[\mu(A) - \mu(A')] \cdot \mathbf{1}_{\{A^* = A \}} \right] \geq 0$.

Observe that the functions $(\mu(A) - \mu(A'))_{+}$ and $\mathbf{1}_{\{A^* = A\}}$ are co-monotone in each coordinate of $\theta$. That is, they are both increasing in some coordinates while both decreasing in other coordinates. Then, we can apply the mixed-monotonicity Harris inequality (see \Cref{remark:mixed-monotonicity-harris}) and obtain:
\begin{align}
    \E \sbr{(\mu(A) - \mu(A')) \cdot \mathbf{1}_{\{A^* = A\}}} = \E \sbr{(\mu(A) - \mu(A'))_{+} \cdot \mathbf{1}_{\{A^* = A\}} } \geq \eps_{\TS} \cdot \delta_A
\end{align}
where $\delta_A = \Pr[A^* = A] \geq \delta_{\TS}$.
}
}
Next, to show a lower bound on  $\E \sbr{(\mu(A) - \mu(A')) \cdot \mathbf{1}_{\{A^* = A\}}}$, we will leverage the Harris inequality \Cref{thm:harris-inequality}, which says increasing functions of independent random variables are non-negatively correlated. Observe that the functions $(\mu(A) - \mu(A'))_{+}$ and $\mathbf{1}_{\{A^* = A\}}$ are co-monotone in each coordinate of $\theta$ (\ie either both increasing or both decreasing in a coordinate).
Then, the mixed-monotonicity Harris inequality (see \Cref{thm:harris-inequality}) implies that:
\begin{align}
    \E \sbr{(\mu(A) - \mu(A')) \cdot \mathbf{1}_{\{A^* = A\}}} = \E \sbr{(\mu(A) - \mu(A'))_{+} \cdot \mathbf{1}_{\{A^* = A\}} } \geq \eps_{\TS} \cdot \delta_A
\end{align}
where $\delta_A = \Pr[A^* = A] \geq \delta_{\TS}$.

To finish the proof, we show the expected absolute difference between $\E\supt[\mu(A) - \mu(A')] \cdot \mathbf{1}_{\{A^* = A \}}$ and $(\mu(A) - \mu(A')) \cdot \mathbf{1}_{\{ A^* = A\}}$ is upper bounded by $\eps_{\TS} \cdot \delta_A$. By regrouping and using triangle inequality as well as $\abs{x \cdot \mathbf{1}_{\{A^* = A \}}} = \abs{x} \cdot \mathbf{1}_{\{A^* = A \}}$, we can upper bound this estimation error by sum of  $\E\sbr{\abs{\E\supt\sbr{\mu(A)} - \mu(A)} \cdot \mathbf{1}_{\{A^*=A\}}}$ and $\E\sbr{\abs{\E\supt\sbr{\mu(A')} - \mu(A')} \cdot \mathbf{1}_{\{A^*=A\}}}$.
Since the mean reward of each atom can be estimated by their empirical average, we can apply Bayesian Chernoff (see \Cref{lem:bayesian-chernoff}) and observe that these two terms are $n^{-1/2}_{\TS}$ times $O(1)$-sub-Gaussian random variables. By the sub-Gaussian tail bound (see \Cref{lem:sub-Gaussian-tail-bound}), we upper bound both terms by $O(n^{-1/2}_{\TS} \delta_A \sqrt{\log(\nicefrac{1}{\delta_A})})$. We conclude by using our choices of $n_{\TS}$ and observing that $\delta_{\TS} \leq \delta_A$.
\end{proof}
\begin{proof}[Proof Sketch for Corollary~\ref{cor:ts-sufficient-samples} (full proof in \Cref{appendix:semi-ts-bic-appendix})]
To derive the dependence of $n_{\TS}$ on $d$, we investigate how the prior-dependent constants $\eps_{\TS}$ and $\delta_{\TS}$ depends on $d$. First, we can let $\eps_{\calC}$ be a version of $\eps_{\TS}$ where the $\min$ is taken over all ordered pairs of priors in $\calC$. Since $\calC$ is finite and satisfies the pairwise non-dominance assumption~\eqref{eq:TS-pairwise}, $\eps_{\TS} \geq \eps_{\calC} > 0$. 

By definition, $\delta_{\TS} = \min_{A \in \calA} \Pr[A^* = A]$ is the minimum probability that arm $A$ is the best arm overall. Fix an arm $A$. We observe that the event where arm $A$ is the best arm is more likely than the event where each atom in $A$ is larger than $\tau$, and all other atoms not in $A$ is smaller than $\nicefrac{\tau}{d}$. Hence, we can lower bound $\Pr[A^* = A]$ by $\E \left[\mathbf{1}_{ \{\forall \ell \in A, \theta_\ell \geq \tau\} } \cdot \mathbf{1}_{\{\forall x \notin A, \theta_x \leq \nicefrac{\tau}{d}\}} \right]$.
Since the prior $\calP$ is independent across atoms, we can write the expression above as product of $\E \left[\mathbf{1}_{\{\forall \ell \in A, \theta_\ell \geq \tau \}} \right]$ and $\E \left[\mathbf{1}_{\{\forall x \notin A, \theta_x \leq \nicefrac{\tau}{d}\}} \right]$. As the values $\{\theta_\ell\}_{\ell \in [d]}$ are independent and co-monotone in each in coordinate of $\theta$, repeated application of mixed-monotonicity Harris inequality (see \Cref{remark:mixed-monotonicity-harris}) implies that:
\begin{align*}
    \Pr[A^* = A] &\geq \prod_{\ell \in A} \E \left[ \mathbf{1}_{\{ \theta_\ell \geq \tau \}} \right] \cdot \prod_{x \notin A} \E \left[ \mathbf{1}_{\{\theta_x \leq \nicefrac{\tau}{d} \}} \right]\\
    &= \prod_{\ell \in A} \Pr[\theta_\ell \geq \tau] \cdot \prod_{x \notin A} \Pr[\theta_x \leq \nicefrac{\tau}{d}]\\
    &\geq \prod_{\ell=1}^d \Pr[\theta_\ell \geq \tau] \Pr[\theta_\ell \leq \nicefrac{\tau}{d}]
\end{align*}
By full support assumption~\eqref{eq:TS-full-support}, we define a prior-dependent constant $\rho_{\tau} = \min_{A \in \calA} \Pr[\theta_{\ell} \geq \tau] > 0$. 
Then, by definition of $\rho_{\tau}$ and the non-degeneracy assumption~\eqref{eq:TS-nondegeneracy}, the expression above is lower bounded by $\rho_{\tau}^d \cdot \poly(\nicefrac{d^d}{(\tau)^d}) \cdot \exp(-d (\nicefrac{\tau}{d})^{-\alpha})$.
Plugging this bound into $n_{TS}$, we obtain $n_{TS} = O_{\calC}(d^{3 + \alpha} \log d)$.
\end{proof}

\subsection{The two-arm case with arbitrary correlation}

What if the prior is \emph{not} independent across atoms? Our analysis extends to the case of $K=2$ arms $A,A'$ with arbitrary correlation between the atoms. In fact, we do not assume combinatorial semi-bandit structure, and instead focus on the fundamental special case of incentivized exploration: when one has two
arms $A,A'$ with arbitrary joint prior on $(\mu(A),\,\mu(A'))$.
\footnote{Equivalently, we have $d=2$ atoms with an arbitrary joint prior on $(\theta_1,\theta_2)$, and the feasible arms are the two singleton arms $\{1\}$ and $\{2\}$.}
We prove that \Cref{thm:semi-ts-bic} extends to this scenario. The analysis is very similar, and omitted.

\begin{theorem}
The assertion in \Cref{thm:semi-ts-bic} holds for the case when one has two arms $A,A'$ and an arbitrary joint prior on $(\mu(A),\,\mu(A'))$.
\end{theorem}

This result completes our understanding of incentivized exploration with two correlated arms: indeed, a necessary and sufficient condition (and the algorithm) are known for collecting the initial data \citep{ICexploration-ec15}. A similar result for two \emph{independent} arms is in \citep{Selke-PoIE-ec21}.

\OMIT{ %%%%
To provide more intuition behind the number of samples $n_{\TS}$ in \Cref{thm:semi-ts-bic}, we provide a corollary demonstrating $n_{\TS}$ dependence on $d$, the dimension of the arm space. First, we posit that the arm set $\calA$ is \emph{non-degenerate}. That is, the arms in $\calA$ are allowed to have different size, but their sizes are within a constant factor of each other. Formally, we have the following assumption: for absolute constants $c_1, c_2$,
\begin{align}
    \label{ass:non-degeneracy}
    \forall A, A' \in \calA, \nrm{A}_{1} \leq c_1 \nrm{A'}_{1} \quad \text{and} \quad \nrm{A - A'}_{1} \leq c_2 \nrm{A}_{1}
\end{align}
As a special case when all arms in $\calA$ have the same size, the dependence on $d$ becomes smaller, and we have $n_{\TS} = O_{\calC}(d^3)$.
} %%%

\section{BIC algorithms for initial exploration}
\label{sec:initial-samples}

\newcommand{\PropHE}{(P)\xspace}

We present two BIC algorithms for \emph{initial exploration}, where the objective is to sample each atom at least once  (\ie choose arms whose union is $[d]$) and complete in $N_0$ rounds for some $N_0$ determined by the prior. Such algorithms are interesting in their own right, and can be used to bootstrap Thompson Sampling as per \Cref{thm:semi-ts-bic}. (To collect $n$ samples of each arm, repeat the algorithm $n$ times.) Both algorithms complete in the number of rounds that is exponential in $\poly(d)$. The first algorithm completes in $\exp(O_\calP(d))$ rounds, but is restricted to arms of the same size and Beta-Bernoulli priors. We obtain $\exp(O_\calP(d^2))$ for arbitrary sets of arms. The second algorithm sidesteps the Beta-Bernoulli restriction, but completes in $\exp(O_\calP(d^3))$ rounds.

\subsection{Reduction to $K$-armed bandits}
\label{subsec:cyclic}

The first algorithm builds on the approach from \citet{ICexploration-ec15}, which is defined for $K$-armed bandits and explores a given sequence of arms as long as a certain property \PropHE holds for this sequence. This property is proved to hold for arms with independent priors, ordered by their prior mean rewards. However, for combinatorial semi-bandits the arms' priors are highly correlated, and satisfying \PropHE is non-trivial. Our technical contribution here to provide a sequence of arms and prove that \PropHE holds. Note that it suffices to explore a sequence of arms which collectively cover all the atoms.

Throughout this subsection, we make the following assumptions:
\begin{align}
&\text{the prior $\calP_\ell$ for each atom $\ell$ is a Beta distribution with parameters $(\alpha_\ell,\beta_\ell)$;} \label{eq:cyclic-assn-2}\\
&\text{the reward distributions $\calD_\ell$ are Bernoulli distributions.} \label{eq:cyclic-assn-3}
\end{align}
This is a paradigmatic special case for Thompson Sampling (and Bayesian inference in general).

%$\kappa:= \Cel{d/m}$
\iffalse
Given any number $n$, define a sequence of $\kappa(n)$ arms $V_1(n) \LDOTS V_{\kappa_n}(n)\in \calA$ such that
\begin{align}
    %V_i = \cbr{ (i-1)m+\ell \pmod d  : \ell\in [m] }, \quad i\in[\kappa].
    V_i = \cbr{ ((i-1)m+\ell-1) \pmod d  +1 : \ell\in [m] }, \quad i\in[\kappa].
\end{align}
In other words, each arm $V_i$ is a subset of $m$ consecutive atoms, starting from atom $1+(i-1)m$, where atom $\ell>d$ is identified with atom $1+ (\ell-1){\pmod d}$. Recall that atoms $1\LDOTS d$ are ordered by their prior mean reward, from largest to smallest.
\fi

Let
    $\nu_\ell(n) = \alpha_\ell\,/\,(\alpha_\ell+\beta_\ell+n)$, $n\geq 0$
be the posterior mean reward of atom $\ell$ when conditioned on $n$ samples of this atom such that each of these samples returns reward $0$. %\ascomment{somebody double-check!}

Given any number $n\in \N$, let us define a sequence of $\kappa(n)\leq \infty$ arms
    $V_1^n \LDOTS V_{\kappa(n)}^n \in \calA$.
Let $V_1$ be a prior-best arm: any arm with the largest prior mean reward. The subsequent arms are defined inductively. Essentially, we pretend that each atom in each arm in the sequence so far has been sampled exactly $n$ times and received $0$ each time it has been sampled. The next arm is defined as the posterior-best arm: an arm with a largest posterior reward after seing these samples. Formally, for each $i\geq 2$, we define arm $V_i^n$ given the previous arms $V_1^n \LDOTS V_{i-1}^n$. For each atom $\ell\in [d]$ define $Z_i^n(\ell)=n$ if this atom is contained in one of the previous arms in the sequence,
%$V_j^n, \forall j\in [i-1]$,
and set $Z_i^n(\ell)=0$ otherwise. Then, define $V_i^n$ as a the posterior-best arm if the posterior mean rewards for atoms $\ell$ are given by $\nu_\ell\rbr{Z_i^n(\ell)}$. That is:
\begin{align}
V_i^n \in \argmax_{A\in\calA} \sum_{\ell\in A} \nu_\ell\rbr{Z_i^n(\ell)}.
\end{align}

The sequence stops when the arms therein cover all atoms at least once, and continues infinitely otherwise; this defines $\kappa(n)$.
%\footnote{For some $n \in \N$, $\kappa(n)$ could be infinite. The sequence stops when the arms therein cover all atoms at least once, and continue indefinitely otherwise. %\vhedit{The sequence stops when the arms therein cover all atoms at least once; this defines $\kappa(n)$. For some $n \in \N$, the arms in the sequence never cover all atoms at least once, and the sequence continue indefinitely; $\kappa(n)$ is infinite in this case.}
%\vhedit{The sequence stops when the arms therein cover all atoms at least once and continue indefinitely otherwise, in which case $\kappa(n)$ is infinite; this defines $\kappa(n)$.}
\footnote{In \Cref{thm-P2-beta} and \Cref{thm-P2-gen}, we upper-bound $\kappa(n)$ for some prior-dependent $n=n_\calP$.}

To state the requisite property \PropHE, we focus on this sequence for a particular, prior-dependent choice of $n$.
%\dndelete{Let $H_i^n$, $i\in[\kappa(n)]$ be a dataset that consists of exactly $n$ samples of each arm $V_1^n \LDOTS V_i^n$; each sample contains the reward for each atom in the respective arm. Let $H_0^n$ be an empty dataset. }
%\vhedit{For any $N\ge n$, let $H_i^N$, $i\in[\kappa(n)]$ be a dataset that consists of exactly $N$ samples of each arm $V_1^n \LDOTS V_i^n$; each sample contains the reward for each atom in the respective arm. Let $H_0^N$ be an empty dataset.}

\begin{itemize}

\item[\PropHE \namedlabel{prop:general-necessary-condition}{\PropHE}]
There exist numbers $\myNp\in\N $ and $\tau_{\calP}, \rho_\calP\in (0,1)$, determined by the prior $\calP$, which satisfy the following. Focus on the sequence of arms
    $V_1 \LDOTS V_\kappa$,
where $\kappa = \kappa(\myNp)$ and $V_i = V_i^{\myNp}$ for each $i\in[\kappa]$.
Let $H_i^N$, $i\in[\kappa]$ be a dataset that consists of exactly $N\in\N$ samples of each arm $V_1 \LDOTS V_i$, where each sample contains the reward for each atom in the respective arm, and $H_0^N$ is an empty dataset. Then
\begin{align}\label{eq:P2}
    \Pr\sbr{X_i^N\ge\tau_\calP}\ge \rho_\calP
    \quad \forall i\in [\kappa] \text{ and } N\ge \myNp,
\end{align}
where the random variable $X_i^N$ is defined as
    \[ X_i^N=\min_{\text{arms } A\neq V_i} \E\sbr{\mu(V_i)-\mu(A)\mid H_{i-1}^N}.\]
\end{itemize}
%We define a number sequence $n_\calP$ is a \textbf{feasible} number sequence of the prior $\calP$ if there exist $\tau_\calP$ and $\rho_\calP$ such that Eq.~\ref{eq:P2} holds.

In intuition, any given arm $V_i$ can be the posterior best arm with a margin $\tau_{\calP}$ and probability at least $\rho_{\calP}$ after seeing at least $\myNp$ samples of the previous arms $V_1 \LDOTS
V_{i-1}$.

Given Property \PropHE, prior work guarantees the following (without relying on assumptions (\ref{eq:cyclic-assn-2}-\ref{eq:cyclic-assn-3}).

%\Cref{alg:cyclic-exploration}

\begin{theorem}[\citet{ICexploration-ec15}]\label{thm-p2-bic}
Assume Property \PropHE holds with constants $\myNp, \tau_\calP, \rho_\calP$ and $\kappa=\kappa(\myNp)$. Then there exists a BIC algorithm which explores each arm $V_1 \LDOTS V_\kappa$ at least $\myNp$ times and completes in $T_0$ rounds, where
    $T_0 = \kappa\cdot \myNp\cdot (1+d) \,/\, (\tau_{\calP} \cdot \rho_{\calP})$.
\end{theorem}

Next, we establish \PropHE. First we state a result for a paradigmatic case when all arms have the same cardinality, then relax it in what follows (with a somewhat weaker guarantee).

\begin{theorem}\label{thm-P2-beta}
Assume Beta-Bernoulli priors (\ref{eq:cyclic-assn-2}-\ref{eq:cyclic-assn-3}). Further, assume that
\begin{align}
\text{The arms are all subsets of $[d]$ of a given size $m$;} \label{eq:cyclic-assn-1}
\end{align}
Then Property \PropHE holds with $\kappa = \kappa(\myNp)=\Cel{d/m}$ and
\begin{align}
\myNp
    &=\Cel{\beta_d/\alpha_d}
    \cdot \max_{\ell\in [d]} \Cel{\alpha_{\ell}}
    \label{eq:number-sequence-beta-frac}\\
\tau_\calP
    &= \min_{\text{atoms } \ell\neq \ell' \in [d],\; n, n'\in \{0,\, \myNp\}}
        \abs{ \nu_{\ell}(n) - \nu_{\ell'}(n')}.
    \label{eq:min-gap} \\
\rho_\calP
    &=  (1-\theta_1^0)^{d\cdot \myNp},
    \label{eq:positive-prob}
\end{align}
as long as $\tau_\calP$ and $\rho_\calP$ are strictly positive.
\end{theorem}

\begin{proof}[Proof Sketch]
For each arm $V_i$, $i\in[\kappa]$ we consider the event that the dataset $H_i^n$ from Property \PropHE contains the reward of $0$ for each samples of each atom. We take $n$ to be large enough so that this event makes all arms $V_1 \LDOTS V_{i-1}$ look inferior to $V_i$, in terms of the posterior mean reward. The key is to lower-bound the probability of this event; a non-trivial step here requires Harris inequality.
\end{proof}

We show that $T_0$, the requisite number of rounds, depends exponentially on the number of atoms $d$. To this end, we define a suitable parameterization of the priors. To handle $\tau_\calP$ in \Cref{thm-P2-beta}, we posit a lower bound that depends on $d$, but this dependence is very mild.

\begin{corollary}\label{cor-p2-beta}
Assume Beta-Bernoulli priors (\ref{eq:cyclic-assn-2}-\ref{eq:cyclic-assn-3}) and that \eqref{eq:cyclic-assn-1} holds. Fix some absolute constants $c_0\in\N$ and $c,c'\in(0,1)$. Suppose $\E[\theta_\ell]\leq c'$ for all atoms, and the priors satisfy the following non-degeneracy conditions:
\begin{align*}
\max_{\ell, \ell'\in[d]} \Cel{\beta_{\ell}/\alpha_{\ell}}\cdot \Cel{\alpha_{\ell'}}&\le c_0, \\
\min_{\ell, \ell'\in [d],\; n, n' \in \{0,\, c_0\} }|\nu_\ell(n) - \nu_{\ell'}(n')|&\geq \Omega(c^{-d})
\end{align*}
Then there exists a BIC algorithm which samples each atom at least once and completes in
    \[ N_0 = O\rbr{c_0\, d \cdot \Phi^d}\]
rounds, where
    $\Phi = c\cdot (1-c')^{-n}$
is a constant.
\end{corollary}
\iffalse
\vhcomment{add a remark as below?}
$\kappa(\myNp)=\Cel{d/m}$ the first $\Cel{d/m}-1$ arms is as follows except for the last arm depending on the prior:
%Given any number $n$, define a sequence of $\kappa(n)$ arms $V_1(n) \LDOTS V_{\kappa_n}(n)\in \calA$ such that
\begin{align}
    %V_i = \cbr{ (i-1)m+\ell \pmod d  : \ell\in [m] }, \quad i\in[\kappa].
    V_i = \cbr{ ((i-1)m+\ell-1) \pmod d  +1 : \ell\in [m] }, \quad i\in[\kappa].
\end{align}
In other words, each arm $V_i^\myNp$ is a subset of $m$ consecutive atoms, starting from atom $1+(i-1)m$, where atom $\ell>d$ is identified with atom $1+ (\ell-1){\pmod d}$. Recall that atoms $1\LDOTS d$ are ordered by their prior mean reward, from largest to smallest.
\fi

\vspace{2mm}

Finally, we handle general feasible sets, \ie without assumption \eqref{eq:cyclic-assn-1}. The guarantee becomes slightly weaker, in that we have $d^2$ in the exponent rather than $d$.

\begin{theorem}\label{thm-P2-gen}
Assume Beta-Bernoulli priors (\ref{eq:cyclic-assn-2}-\ref{eq:cyclic-assn-3}).
Then Property \PropHE holds with $\kappa(\myNp)\le d$ and
\begin{align}
\myNp
    &=\Cel{(\alpha_d+\beta_d)/\alpha_d}
    \cdot \max_{\ell\in [d]} \Cel{\alpha_{\ell}}
\cdot d    \label{eq:number-sequence-beta-frac-gen}\\
\tau_\calP
    &=\min_{A\neq A'\in \calA, n\in\{0, \myNp\}^d} \abs{\sum_{\ell\in A}\nu_\ell(n_\ell)-\sum_{\ell'\in A'}\nu_{\ell'}(n_{\ell'})}.\label{eq:min-gap-gen}\\
\rho_\calP
    &=  (1-\theta_1^0)^{d\cdot \myNp},
    \label{eq:postive-prob-gen}
\end{align}
as long as $\tau_\calP$ and $\rho_\calP$ are strictly positive.
\end{theorem}

\begin{corollary} \label{cor-p2-gen}
Assume Beta-Bernoulli priors (\ref{eq:cyclic-assn-2}-\ref{eq:cyclic-assn-3}).
Fix some absolute constants $c_1\in\N$ and $c_2,c_3\in(0,1)$.
%Suppose the per-atom priors $\calP_\ell$, $\ell\in[d]$ satisfy the following non-degeneracy conditions:
Suppose $\E\sbr{\theta_\ell}\le c_3$ for all atoms, and the priors satisfy the following non-degeneracy conditions: %for all atoms $\ell,\ell'\in [d]$,
\begin{align}
\max_{\text{atoms }\ell, \ell'\in[d]}
    \Cel{(\alpha_\ell+\beta_{\ell})/\alpha_{\ell}}\cdot \Cel{\alpha_{\ell'}}
    \leq c_1 \nonumber \\
    %\label{eq:cor-p2-gen-1} \\
%\min_{n\in\{0, c_1 d\}^d}
    %\tau_\calP
\min_{A\neq A'\in \calA, n\in\{0, \myNp\}^d} \abs{\sum_{\ell\in A}\nu_\ell(n_\ell)-\sum_{\ell'\in A'}\nu_{\ell'}(n_{\ell'})}
    \geq \Omega(c_2^{-d^2}). \label{eq:cor-p2-gen-2}
\end{align}
 %$\min_{n, n'\in \{0,\, c_1 d\} }(|\nu_\ell(n) - \nu_{\ell'}(n')|)\geq \Omega(c_2^{-d^2})$
%    $\min_{A\neq A'\in \calA, n\in\{0, c_1 d\}^d} \abs{\sum_{\ell\in A}\nu_\ell(n_\ell)-\sum_{\ell'\in A'}\nu_{\ell'}(n_{\ell'})} \geq \Omega(c_2^{-d^2})$
%for all atoms $\ell,\ell'$.

%Further, the prior mean rewards for all atoms are $\leq c_3$.
Then there exists a BIC algorithm which samples each atom at least once and completes in
    \[ N_0 = O\rbr{c_1\, d^3 \cdot \Phi^{d^2}}\]
rounds, where
    $\Phi = c_2\cdot (1-c_3)^{-c_1}$
is a constant.
\end{corollary}

In \Cref{sec:mild-assumption}, we provide some motivation for why \eqref{eq:cor-p2-gen-2} is a mild assumption. Our intuition is that ``typically" $\tau_\calP$ should be on the order of $e^{-O(d)}$, whereas \eqref{eq:cor-p2-gen-2} only requires it to be $\geq e^{-\Omega(d^2)}$.

\subsection{Reduction to incentivized reinforcement learning}
\label{subsec:hh}

Our second algorithm builds on the Hidden Hallucination approach from \citet{IncentivizedRL}, which targets incentivized exploration for episodic reinforcement learning. We use this approach by ``encoding" a problem instance of combinatorial semi-bandits as a tabular MDP, so that actions in the MDP correspond to atoms, and feasible trajectories correspond to feasible arms. Then we invoke a theorem in \citet{IncentivizedRL} and ``translate" this theorem back to combinatorial semi-bandits.

ore specifically, consider a tabular MDP with deterministic transitions and unique initial state. Each action in this MDP correspond to some atom $\ell$; then the action's reward is drawn from the corresponding reward distribution $\calD_\ell$. In general, only a subset of actions is feasible at a given state-stage pair of the MDP. Let $G$ be the transition graph in such MDP: it is a rooted directed graph such that the nodes of $G$ correspond to state-stage pairs in the MDP (the root node corresponding to the initial state and stage $0$). Each edge $(u,v)$ in $G$ corresponds to some MDP action feasible at $u$, \ie to some atom. While different edges in $G$ can correspond to the same atom, we require that any rooted directed path in $G$ cannot contain two edges that correspond to the same atom. Let $A_P$ be the subset of atoms that corresponds to a given rooted directed path $P$, and let
    $\calA_G = \cbr{ A_P: \text{rooted directed paths in $G$}}$
be the family of arms ``encoded" by $G$. A family of arms $\calA$ is called \emph{MDP-encodable} if $\calA = \calA_G$ for some transition graph $G$ as defined above with $O(d^2)$ nodes.

Our result applies to all MDP-encodable feasible sets. In particular, the set of all subsets of exactly $m$ atoms, for some fixed $m\leq d$, is MDP-encodable. To see this, consider an MDP with $m$ stages and $d$ states, where each state $\ell\in[d]$ corresponds to the largest atom already included in the arm, and actions feasible at a given stage $i$ and state $\ell$ correspond to all atoms larger than $\ell$.

Our result allows for arbitrary per-atom priors $\calP_\ell$, subject to a minor non-degeneracy condition, and reward distributions $\calD_\ell$ that are supported on the same countable set.

\begin{theorem}\label{thm:our-HH}
Consider a feasible arm set $\calA$ that is MDP-encodable, as defined above. Suppose the per-atom priors $\calP_\ell$ lie in some fixed, finite collection $\calC$ such that any $\calP_\ell\in \calC$ satisfies
    $\Pr[\theta_\ell\leq \eps]>0$ for all $\eps>0$
and
    $\E[\theta_\ell]>0$.
Further, suppose all reward distributions $\calD_\ell$ that are supported on the same countable set.
Fix parameter $\delta \in (0,1)$. There is a BIC algorithm such that with probability at least $1-\delta$ each atom is sampled at least once. This algorithm completes in $N_0$ rounds, where
%\begin{equation}\label{eq:thm:hh-guarantee-arbitrary}
    $N_0 = \Phi_{\calC}^{-d^3} \cdot O_{\calC}\rbr{\poly(d) \cdot \log(\delta^{-1})}$
%\end{equation}
for some constant $\Phi_{\calC} \in (0,1)$ determined by collection $\calC$.
\label{thm:hh-guarantee-arbitrary}
\end{theorem}

\begin{remark} While the guarantee in \Cref{thm:hh-guarantee-arbitrary} holds with probability $1-\delta$, rather than almost surely, it suffices to bootstrap Thompson Sampling in \Cref{thm:semi-ts-bic}. To see this, let $\ALG$ be an algorithm that runs for
    $T_0 = N_0\cdot n_{\TS}+d\cdot n_{\TS}$
rounds (with $n_{\TS}$ from \Cref{thm:semi-ts-bic}), and proceeds as follows: first it repeats the algorithm from \Cref{thm:hh-guarantee-arbitrary} $n_{\TS}$ times, and in the remaining rounds it deterministically plays an arm with the largest prior mean reward (this algorithm is BIC). Define the ``success event" as one in which \ALG samples each atom $\geq n_{\TS}$ times in the first $N_0\cdot N_{\TS}$ rounds. Consider \emph{another} algorithm, $\ALG^*$, which runs for $T_0$ rounds, coincides with $\ALG$ on the first $N_0\cdot N_{\TS}$ rounds, and in the remaining rounds coincides with $\ALG$ on the success event, and otherwise plays some arms so as to sample each atom at least once (this algorithm is not necessarily BIC). Now, if Thompson Sampling is preceded by $\ALG^*$, then the analysis in \Cref{thm:semi-ts-bic} guarantees that each round $t$ of Thompson Sampling satisfies the BIC property \eqref{def:BIC}, and does so with a strictly positive prior-dependent constant on the right hand side of \eqref{def:BIC}. Therefore, the same holds for $\ALG$, since it coincides with $\ALG^*$ w.h.p., if the failure probability $\delta$ in \Cref{thm:hh-guarantee-arbitrary} is chosen small enough.
\end{remark}

%\newpage
\bibliographystyle{abbrvnat}
\bibliography{sample,bib-abbrv,bib-slivkins,bib-bandits,bib-AGT}

\begin{thebibliography}{31}
\providecommand{\natexlab}[1]{#1}
\providecommand{\url}[1]{\texttt{#1}}
\expandafter\ifx\csname urlstyle\endcsname\relax
  \providecommand{\doi}[1]{doi: #1}\else
  \providecommand{\doi}{doi: \begingroup \urlstyle{rm}\Url}\fi

\bibitem[Agrawal and Goyal(2017)]{Shipra-aistats13-JACM}
S.~Agrawal and N.~Goyal.
\newblock Near-optimal regret bounds for thompson sampling.
\newblock \emph{J. of the ACM}, 64\penalty0 (5):\penalty0 30:1--30:24, 2017.
\newblock Preliminary version in \emph{AISTATS 2013}.

\bibitem[Bahar et~al.(2016)Bahar, Smorodinsky, and Tennenholtz]{Bahar-ec16}
G.~Bahar, R.~Smorodinsky, and M.~Tennenholtz.
\newblock Economic recommendation systems.
\newblock In \emph{16th ACM Conf. on Electronic Commerce (ACM-EC)}, 2016.

\bibitem[Bahar et~al.(2019)Bahar, Smorodinsky, and Tennenholtz]{Bahar-ec19}
G.~Bahar, R.~Smorodinsky, and M.~Tennenholtz.
\newblock Social learning and the innkeeper's challenge.
\newblock In \emph{ACM Conf. on Economics and Computation (ACM-EC)}, pages
  153--170, 2019.

\bibitem[Bergemann and Morris(2019)]{BergemannMorris-survey19}
D.~Bergemann and S.~Morris.
\newblock Information design: A unified perspective.
\newblock \emph{Journal of Economic Literature}, 57\penalty0 (1):\penalty0
  44--95, March 2019.

\bibitem[Bimpikis et~al.(2018)Bimpikis, Papanastasiou, and
  Savva]{Bimpikis-exploration-ms17}
K.~Bimpikis, Y.~Papanastasiou, and N.~Savva.
\newblock Crowdsourcing exploration.
\newblock \emph{Management Science}, 64\penalty0 (4):\penalty0 1477--1973,
  2018.

\bibitem[Castiglioni et~al.(2020)Castiglioni, Celli, Marchesi, and
  Gatti]{DBLP:conf/nips/CastiglioniCM020}
M.~Castiglioni, A.~Celli, A.~Marchesi, and N.~Gatti.
\newblock Online bayesian persuasion.
\newblock In \emph{Advances in Neural Information Processing Systems
  (NeurIPS)}, 2020.

\bibitem[Chen et~al.(2018)Chen, Frazier, and Kempe]{Kempe-colt18}
B.~Chen, P.~I. Frazier, and D.~Kempe.
\newblock Incentivizing exploration by heterogeneous users.
\newblock In \emph{Conf. on Learning Theory (COLT)}, pages 798--818, 2018.

\bibitem[Chen et~al.(2013)Chen, Wang, and Yuan]{Chen-icml13}
W.~Chen, Y.~Wang, and Y.~Yuan.
\newblock Combinatorial multi-armed bandit: General framework and applications.
\newblock In \emph{20th Intl. Conf. on Machine Learning (ICML)}, pages
  151--159, 2013.

\bibitem[Degenne and Perchet(2016)]{degenne2016combinatorial}
R.~Degenne and V.~Perchet.
\newblock Combinatorial semi-bandit with known covariance.
\newblock In D.~Lee, M.~Sugiyama, U.~Luxburg, I.~Guyon, and R.~Garnett,
  editors, \emph{Advances in Neural Information Processing Systems}, volume~29.
  Curran Associates, Inc., 2016.
\newblock URL
  \url{https://proceedings.neurips.cc/paper/2016/file/e816c635cad85a60fabd6b97b03cbcc9-Paper.pdf}.

\bibitem[Frazier et~al.(2014)Frazier, Kempe, Kleinberg, and
  Kleinberg]{Frazier-ec14}
P.~Frazier, D.~Kempe, J.~M. Kleinberg, and R.~Kleinberg.
\newblock Incentivizing exploration.
\newblock In \emph{ACM Conf. on Economics and Computation (ACM-EC)}, 2014.

\bibitem[Gopalan et~al.(2014)Gopalan, Mannor, and
  Mansour]{gopalan2014ThompsonSF}
A.~Gopalan, S.~Mannor, and Y.~Mansour.
\newblock Thompson sampling for complex online problems.
\newblock In \emph{ICML}, 2014.

\bibitem[Gy{\"{o}}rgy et~al.(2007)Gy{\"{o}}rgy, Linder, Lugosi, and
  Ottucs{\'{a}}k]{Gyorgy-jmlr07}
A.~Gy{\"{o}}rgy, T.~Linder, G.~Lugosi, and G.~Ottucs{\'{a}}k.
\newblock The on-line shortest path problem under partial monitoring.
\newblock \emph{J. of Machine Learning Research (JMLR)}, 8:\penalty0
  2369--2403, 2007.

\bibitem[Harris(1960)]{harris_1960}
T.~E. Harris.
\newblock A lower bound for the critical probability in a certain percolation
  process.
\newblock \emph{Mathematical Proceedings of the Cambridge Philosophical
  Society}, 56\penalty0 (1):\penalty0 13–20, 1960.
\newblock \doi{10.1017/S0305004100034241}.

\bibitem[Immorlica et~al.(2020)Immorlica, Mao, Slivkins, and
  Wu]{Jieming-unbiased18}
N.~Immorlica, J.~Mao, A.~Slivkins, and S.~Wu.
\newblock Incentivizing exploration with selective data disclosure.
\newblock In \emph{ACM Conf. on Economics and Computation (ACM-EC)}, 2020.
\newblock Working paper available at {\tt https://arxiv.org/abs/1811.06026}.

\bibitem[Kamenica(2019)]{Kamenica-survey19}
E.~Kamenica.
\newblock Bayesian persuasion and information design.
\newblock \emph{Annual Review of Economics}, 11\penalty0 (1):\penalty0
  249--272, 2019.

\bibitem[Kaufmann et~al.(2012)Kaufmann, Korda, and Munos]{Kaufmann-alt12}
E.~Kaufmann, N.~Korda, and R.~Munos.
\newblock Thompson sampling: An asymptotically optimal finite-time analysis.
\newblock In \emph{23rd Intl. Conf. on Algorithmic Learning Theory (ALT)},
  pages 199--213, 2012.

\bibitem[Kremer et~al.(2014)Kremer, Mansour, and Perry]{Kremer-JPE14}
I.~Kremer, Y.~Mansour, and M.~Perry.
\newblock Implementing the ``wisdom of the crowd".
\newblock \emph{J. of Political Economy}, 122\penalty0 (5):\penalty0 988--1012,
  2014.
\newblock Preliminary version in \emph{ACM EC 2013}.

\bibitem[Kveton et~al.(2014)Kveton, Wen, Ashkan, Eydgahi, and
  Eriksson]{MatroidBandits-uai14}
B.~Kveton, Z.~Wen, A.~Ashkan, H.~Eydgahi, and B.~Eriksson.
\newblock Matroid bandits: Fast combinatorial optimization with learning.
\newblock In \emph{13th Conf. on Uncertainty in Artificial Intelligence (UAI)},
  pages 420--429, 2014.

\bibitem[Kveton et~al.(2015)Kveton, Wen, Ashkan, and
  Szepesv{\'{a}}ri]{Kveton-aistats15}
B.~Kveton, Z.~Wen, A.~Ashkan, and C.~Szepesv{\'{a}}ri.
\newblock Tight regret bounds for stochastic combinatorial semi-bandits.
\newblock In \emph{18th Intl. Conf. on Artificial Intelligence and Statistics
  (AISTATS)}, 2015.

\bibitem[Mansour et~al.(2020)Mansour, Slivkins, and
  Syrgkanis]{ICexploration-ec15}
Y.~Mansour, A.~Slivkins, and V.~Syrgkanis.
\newblock Bayesian incentive-compatible bandit exploration.
\newblock \emph{Operations Research}, 68\penalty0 (4):\penalty0 1132--1161,
  2020.
\newblock Preliminary version in \emph{ACM EC 2015}.

\bibitem[Mansour et~al.(2022)Mansour, Slivkins, Syrgkanis, and
  Wu]{ICexplorationGames-ec16}
Y.~Mansour, A.~Slivkins, V.~Syrgkanis, and S.~Wu.
\newblock Bayesian exploration: Incentivizing exploration in {B}ayesian games.
\newblock \emph{Operations Research}, 2022.
\newblock To appear. Preliminary version in \emph{ACM EC 2016}.

\bibitem[Russo and {Van Roy}(2014)]{Russo-MathOR-14}
D.~Russo and B.~{Van Roy}.
\newblock Learning to optimize via posterior sampling.
\newblock \emph{Mathematics of Operations Research}, 39\penalty0 (4):\penalty0
  1221--1243, 2014.

\bibitem[Russo and {Van Roy}(2016)]{Russo-JMLR-16}
D.~Russo and B.~{Van Roy}.
\newblock An information-theoretic analysis of thompson sampling.
\newblock \emph{J. of Machine Learning Research (JMLR)}, 17:\penalty0
  68:1--68:30, 2016.

\bibitem[Russo et~al.(2018)Russo, {Van Roy}, Kazerouni, Osband, and
  Wen]{TS-survey-FTML18}
D.~Russo, B.~{Van Roy}, A.~Kazerouni, I.~Osband, and Z.~Wen.
\newblock A tutorial on thompson sampling.
\newblock \emph{Foundations and Trends in Machine Learning}, 11\penalty0
  (1):\penalty0 1--96, 2018.
\newblock Published with \emph{Now Publishers} (Boston, MA, USA). Also
  available at {\tt https://arxiv.org/abs/1707.02038}.

\bibitem[Sellke and Slivkins(2021)]{Selke-PoIE-ec21}
M.~Sellke and A.~Slivkins.
\newblock The price of incentivizing exploration: A characterization via
  thompson sampling and sample complexity.
\newblock In \emph{22th ACM Conf. on Economics and Computation (ACM-EC)}, 2021.

\bibitem[Simchowitz and Slivkins(2021)]{IncentivizedRL}
M.~Simchowitz and A.~Slivkins.
\newblock Incentives and exploration in reinforcement learning, 2021.
\newblock Working paper, available at
  \texttt{https://arxiv.org/abs/2103.00360}.

\bibitem[Slivkins(2019)]{slivkins-MABbook}
A.~Slivkins.
\newblock Introduction to multi-armed bandits.
\newblock \emph{Foundations and Trends$\circledR$ in Machine Learning},
  12\penalty0 (1-2):\penalty0 1--286, Nov. 2019.
\newblock Published with \emph{Now Publishers} (Boston, MA, USA). Also
  available at {\tt https://arxiv.org/abs/1904.07272}. Latest online revision:
  Jan 2022.

\bibitem[Thompson(1933)]{Thompson-1933}
W.~R. Thompson.
\newblock {On the likelihood that one unknown probability exceeds another in
  view of the evidence of two samples.}
\newblock \emph{Biometrika}, 25\penalty0 (3-4):\penalty0 285--294, 1933.

\bibitem[Wang and Chen(2018)]{wang2018tscb}
S.~Wang and W.~Chen.
\newblock Thompson sampling for combinatorial semi-bandits.
\newblock In J.~Dy and A.~Krause, editors, \emph{Proceedings of the 35th
  International Conference on Machine Learning}, volume~80 of \emph{Proceedings
  of Machine Learning Research}, pages 5114--5122. PMLR, 10--15 Jul 2018.
\newblock URL \url{https://proceedings.mlr.press/v80/wang18a.html}.

\bibitem[Wen et~al.(2015)Wen, Kveton, and Ashkan]{Wen-icml15}
Z.~Wen, B.~Kveton, and A.~Ashkan.
\newblock Efficient learning in large-scale combinatorial semi-bandits.
\newblock In \emph{32nd Intl. Conf. on Machine Learning (ICML)}, pages
  1113--1122, 2015.

\bibitem[Zu et~al.(2021)Zu, Iyer, and Xu]{DBLP:conf/sigecom/ZuIX21}
Y.~Zu, K.~Iyer, and H.~Xu.
\newblock Learning to persuade on the fly: Robustness against ignorance.
\newblock In \emph{ACM Conf. on Economics and Computation (ACM-EC)}, pages
  927--928, 2021.

\end{thebibliography}
\newpage
\newpage
\appendix
\numberwithin{equation}{section}
% {\par\centering\LARGE\bf Appendices for Incentivizing Combinatorial Bandit Exploration\par}
% \vskip 0.4in
% \input{./appendix/app-errata-neurips22subm}
\section{BIC analysis for Thompson Sampling (proofs for \Cref{sec:ts})}
\label{appendix:semi-ts-bic-appendix}

This appendix provides the proofs for \Cref{sec:ts}, the BIC analysis of Thompson Sampling. Specifically, we prove \Cref{thm:semi-ts-bic} (that Thompson Sampling is BIC when initialized with sufficiently many samples) and \Cref{cor:ts-sufficient-samples} (that the sufficient number of samples is polynomial in $d$).

\paragraph{Proof of \Cref{thm:semi-ts-bic}}

By definition, Thompson sampling is BIC at a particular round $t > T_0$ if and only if we have $\textstyle{\E[\mu(A) - \mu(A')| A\supt = A] \geq 0}$ for all $(i,j)$ such that $i \neq j$. This condition can be written as:
\begin{align*}
    \E[\mu(A) - \mu(A') | A\supt = A] &= \frac{\E \left[\E\supt[\mu(A) - \mu(A')]\Pr\supt[A\supt = A] \right]}{\Pr[A\supt = A]} \\
    &= \frac{\E \left[\E\supt[\mu(A) - \mu(A')]\Pr\supt[A^* = A] \right]}{\Pr[A^* = A]} \tag{by definition of Thompson Sampling}
\end{align*}
Observe that the denominator $\Pr[A^* = A]$ is a positive prior-dependent constant. Hence, we only need to bound the numerator to satisfy the BIC condition.

Fixing arms $A, A'$, we can rewrite the numerator as:
\begin{align*}
    \E \left[\E\supt[\mu(A) - \mu(A')] \Pr\supt[A^* = A] \right] &= \E \left[ \E\supt \left[ \E\supt[\mu(A) - \mu(A')] \cdot \mathbf{1}_{\{A^* = A\}} \right]
    \right] \\
    &= \E\left[ \E\supt[\mu(A) - \mu(A')] \cdot \mathbf{1}_{\{A^* = A\}} \right]
\end{align*}
For Thompson sampling to be BIC, it suffices to show that $\E\left[ \E\supt[\mu(A) - \mu(A')] \mathbf{1}_{\{A^* = A\}} \right] \geq 0$.

We first prove our observation that the functions $(\mu(A) - \mu(A'))_{+}$ and $\mathbf{1}_{\{ A^* = A\}}$ are co-monotone in each coordinate of $\theta$, which means they are both increasing in some coordinates while both decreasing in the other coordinates. Specifically for any $\ell$-th coordinate of $\theta$, they are both increasing in  $\theta_\ell$ (given all other coordinates in $\theta$ stay the same) if $A_\ell$ (the $\ell$-th coordinate of $A$) equals to $1$. Otherwise, if $A_\ell=0$, they are both decreasing in $\theta_\ell$.

Given any $\theta$ and $\theta'$ having the same coordinates $\theta_x=\theta_x'$ ($x \in [d]$) except for the $\ell$-th coordinate, $\theta_\ell>\theta_\ell'$, and for any arm $A' \neq A$,
\begin{align}
    &\langle \theta, A-A'\rangle- \langle \theta', A-A'\rangle \\
    &=\sum_{x=1}^d \theta_x(A_x - A'_x)-\sum_{x=1}^d \theta_x'(A_x - A'_x)\\
    &=(\theta_\ell-\theta_\ell')(A_\ell - A'_\ell)\tag{other coordinates are the same except $\ell$ }\\
    &\left\{
        \begin{array}{lr}
        \ge 0, \quad \text{if }A_\ell=1\\
        \le 0, \quad \text{if }A_\ell=0
        \end{array}
    \right. \tag{$\theta_\ell>\theta_\ell'$ and $A_\ell, A'_\ell\in\{0, 1\}$}%\label{eq:co-monotone}
\end{align}
%Eq.~\ref{eq:co-monotone} holds because $\theta_\ell>\theta_\ell'$ and $A_{i, \ell}, A_{l, \ell}\in\{0, 1\}$.
Note that $\mathbf{1}_{\{A^*=A\}}=\mathbf{1}_{\{\mu(A) - \mu(A')\ge 0, \forall A'\neq A\}}$. So if $A_\ell=1$, then $\mu(A) - \mu(A')$ are increasing in $\theta_\ell$ for all $A'\neq A$, especially for $A_k=A'$. Hence, $\mathbf{1}_{\{A^*=A\}}$ and $\mu(A) - \mu(A')$ are both increasing in $\theta_\ell$. Otherwise, if $A_\ell=0$, they are both decreasing in $\theta_\ell$.

Hence, we can apply \Cref{remark:mixed-monotonicity-harris} to lower bound the expression above as follows:
\begin{align*}
    \E[\mu(A) - \mu(A')\cdot \mathbf{1}_{\{ A^* = A\}}] &= \E[(\mu(A) - \mu(A'))_{+} \cdot \mathbf{1}_{\{A^* = A\}}] \\
    &\geq \min_{A, A' \in \calA} \E[(\mu(A) - \mu(A'))_{+}] \Pr[A^* = A] \\
    &= \epsilon_\TS \delta_A
\end{align*}
where $\delta_A = \Pr[A^* = A] \geq \delta_\TS$.

To finish the proof, we need the following inequality to hold:
\begin{align}
    \E[\abs{\E^{T_0} [\langle \theta, A - A' \rangle] \cdot \mathbf{1}_{\{A^* = A\}} - \langle \theta, A - A' \rangle \cdot \mathbf{1}_{\{A^* = A\}}  }] \leq \epsilon_\TS \delta_A
    \label{eq:TS-ineq}
\end{align}
where $\delta_A = \Pr[A^* = A]$.
Regrouping and using triangle inequality on the LHS of \Cref{eq:TS-ineq}, we have:
\begin{align}
    &\quad \E \left[\abs{\E^{T_0} [\langle \theta, A - A' \rangle] \cdot \mathbf{1}_{\{A^* = A\}} - \langle \theta, A - A' \rangle \cdot \mathbf{1}_{\{A^* = A \}}  } \right] \nonumber \\
    &\leq \E \left[\abs{\E^{T_0} [\langle \theta, A \rangle ] - \langle \theta, A \rangle} \cdot \mathbf{1}_{\{ A^* = A \}} \right] + \E \left[\abs{\E^{T_0} [\langle \theta, A' \rangle ] - \langle \theta, A' \rangle} \cdot \mathbf{1}_{\{ A^* = A \}} \right]
    \label{eq:ts-bic-concentration}
\end{align}
The final step is to bound each individual summand in the inequality above. 
%From \Cref{lem:semi-confidence-bound}, we have an estimate $\hat{\theta}$ of the latent parameter $\theta$ such that $\displaystyle{\nrm{\hat{\theta} - \theta}_2 \leq n^{-1/2}\TS \sqrt{2 d \cdot \log(d/\delta)} }$ with probability at least $1-\delta$. Also, 
By the Bayesian Chernoff Bound (\Cref{lem:bayesian-chernoff}), we have $\nrm{\E^{T_0}[\theta] - \theta}$ is a $(n_\TS^{-1/2}\sqrt{d})$ times $O(1)$-sub-Gaussian random variable. Then, by Cauchy-Schwarz inequality, we have
\begin{align*}
    \abs{\E^{T_0}[\langle \theta, A\rangle] - \langle \theta, A \rangle} &\leq \nrm{\E^{T_0}[\theta] - \theta} \cdot \nrm{A}\\
    &\leq \sqrt{d} \nrm{\E^{T_0}[\theta] - \theta}
\end{align*}

Hence, $\abs{\E^{T_0}[\langle \theta, A\rangle] - \langle \theta, A \rangle}$ and $\abs{\E^{T_0}[\langle \theta, A'\rangle] - \langle \theta, A' \rangle}$ have magnitude at most as large as a $(n^{-1/2}_\TS \cdot d)$ times $O(1)$-sub-Gaussian random variable. Then, we can apply \Cref{lem:sub-Gaussian-tail-bound} to both terms in the inequality \eqref{eq:ts-bic-concentration} above and upper bound it by at most $O \left(\delta_A \cdot n^{-1/2}_\TS \cdot d \sqrt{\log \left(\delta_A^{-1} \right)} \right)$. Then, using our choice of $n_\TS$ and $\delta_\TS \leq \delta_A$, we arrive at the conclusion.

\paragraph{Proof of \Cref{cor:ts-sufficient-samples}}
Recall that by \Cref{thm:semi-ts-bic}, we have $n_\TS = C_\TS \cdot d^2 \cdot \epsilon^{-2}_\TS \cdot \log(\delta^{-1}_\TS)$.
Let $\epsilon_{\calC}$ be the version of $\epsilon_\TS$ where the $\min$ is taken over all ordered pairs of priors in $\calC$. Then we have $\epsilon_\TS \geq \epsilon_{\calC}$. Since $\calC$ is finite and satisfy the pairwise non-dominance assumption, $\epsilon_{\calC}$ is strictly positive.

By definition, $\delta_\TS = \min_{A \in \calA} \Pr[A^* = A]$. Fix an arm $A$. Then, we can decompose the probability of arm $A$ being the best arm as:
\begin{align}
    \Pr[A^* = A] &= \Pr[\langle \theta, A - A' \rangle \geq 0, \forall A' \neq A ] \nonumber\\
    &= \Pr \left[\sum_{\ell \in A} \theta_\ell - \sum_{x\in A'} \theta_x \geq 0, \forall A' \neq A \right] \label{eq:pairwise-theta-comparison}
\end{align}
We observe that the event when $A$ is the best arm is more likely than the event when each atom in $A$ is larger than $\tau$, and all other atoms not in $A$ is smaller than $\nicefrac{\tau}{d}$. Hence, we can lower bound the probability above as:
\begin{align*}
   &\quad \Pr \left[\sum_{\ell \in A} \theta_\ell - \sum_{x\in A'} \theta_x \geq 0, \forall A' \neq A  \right]\\
   &\geq \Pr \left[ \forall \ell \in A, \theta_\ell \geq \tau \quad\text{and} \quad \forall x \notin A, \theta_x \leq \nicefrac{\tau}{d} \right] \\
   &=  \Pr \left[\forall \ell \in A, \theta_\ell \geq \tau \right] \cdot \Pr \left[\forall x \notin A, \theta_x \leq \nicefrac{\tau}{d} \right] \tag{the prior is independent across atoms}\\
   &=  \E \left[ \prod_{\ell \in A} \mathbf{1}_{\{\theta_\ell \geq \tau\}} \right] \cdot \E \left[ \prod_{x \notin A} \mathbf{1}_{\{ \theta_x \leq \nicefrac{\tau}{d} \}} \right]
\end{align*}
Observe that the values $\{ \theta_\ell \}_{\ell \in [d]}$ are independent, and each function $\mathbf{1}_{\{\theta_\ell \geq \tau\}}$ (and $\mathbf{1}_{\{\theta_\ell \leq \nicefrac{\tau}{d}\}}$) are co-monotone in each coordinate of $\theta$. Then, repeated application of mixed-monotone Harris inequality (see \Cref{remark:mixed-monotonicity-harris}) implies that
\begin{align*}
    \Pr[A^* = A] &\geq \prod_{\ell \in A}\E[\mathbf{1}_{\{ \theta_\ell \geq \tau \}}] \cdot \prod_{x \notin A}\E[\mathbf{1}_{\{ \theta_x \leq \nicefrac{\tau}{d} \}}] \tag{mixed-monotonicity Harris}\\
   &= \prod_{\ell \in A}\Pr[\theta_\ell \geq \tau] \cdot \prod_{x \notin A}\Pr[\theta_x \leq \nicefrac{\tau}{d}]\\
   &\geq \prod_{\ell=1}^d \Pr[\theta_\ell \geq \tau] \Pr[\theta_\ell \leq \nicefrac{\tau}{d}]
\end{align*}
By the full support assumption \Cref{eq:TS-full-support}, we define a prior-dependent constant $\rho_{\tau} = \min_{\ell \in [d]} \Pr[\theta_\ell \geq \tau] > 0$. Then, by definition of $\rho_{\tau}$ and the non-degeneracy assumption \Cref{eq:TS-nondegeneracy}, we have:
\begin{align*}
    \prod_{\ell=1}^d \Pr[\theta_\ell \geq \tau] \Pr[\theta_\ell \leq \nicefrac{\tau}{d}] &\geq \prod_{\ell=1}^d \rho_{\tau}^d \cdot \poly(\nicefrac{d}{\tau}) \cdot \exp(- (\nicefrac{\tau}{d})^{-\alpha})\\
    &\geq \rho_{\tau}^d \cdot \poly(\nicefrac{d^d}{(\tau)^d}) \cdot \exp(-d (\nicefrac{\tau}{d})^{-\alpha})
\end{align*}
Substituting this expressions and $\epsilon_\TS$ into $n_\TS$, we have $n_\TS = O_{\calC}(d^{3 + \alpha} \log d)$.

\section{Initial exploration: reduction to $K$-armed bandits (proofs for \Cref{subsec:cyclic})}
\subsection{\Cref{thm-p2-bic}: the approach from \citet{ICexploration-ec15}}
%Initial sampling algorithm and BIC guarantee}

Recall that we build on an approach from \citet{ICexploration-ec15}, encapsulated in \Cref{thm-p2-bic}.  Let us clarify how this theorem follows from the material in \citet{ICexploration-ec15}.

%In this subsection, we firstly present an algorithm which explores each arm $V_1 \LDOTS V_\kappa$ at least $N$ times and completes in a predetermined number of rounds. We then prove the algorithm is BIC under \PropHE, namely the proof of \Cref{thm-p2-bic}. Before this proof, we state \Cref{lem:bic-ec15} derived from \citet{ICexploration-ec15} and use the Lemma to prove the Theorem.
%from \citet{ICexploration-ec15} using notations in this paper.

The algorithm from \citet{ICexploration-ec15} is modified in two ways: it explores the arms in the order given by the sequence $V_1 \LDOTS V_\kappa$, and the observed outcome from playing a given arm now includes the rewards for all atoms in this arm. Let us spell out the resulting algorithm, for  completeness.

\begin{algorithm}[H]

\SetKwInput{KwParams}{Parameters}
\KwParams{$L,N\in\N$}

For the first $N$ rounds, recommend arm $V_1$.\\
Let $s_1 = \rbr{ r\supt_\ell:\; \ell\in V_1,\,t\in[N] }$ 
be the tuple of all observed per-atom rewards from arm $V_1$; \\
\For{each arm $V_i$ in increasing order of $i$}
{ Let $A^* = \argmax_{A \in \calA} \E\sbr{ \mu(A)\mid s_1 \LDOTS s_{i-1}}$, breaking ties favoring smaller index;\\

From the set $P$ of the next $L \cdot N$ rounds, pick a set $Q$ of $N$ rounds uniformly at random;\\

Every agent $p \in P - Q$ is recommended arm $A^*$; \\

Every agent $p \in Q$ is recommended arm $V_i$;\\

Let $s_i^{N}$ be the tuple of all per-atom rewards from arm $V_i$ observed in rounds $t\in Q$;
}
\caption{Hidden Exploration (modification of Algorithm 3 in \citet{ICexploration-ec15})}
\label{alg:cyclic-exploration}
\end{algorithm}

The analysis in Section 5.2 of \citet{ICexploration-ec15} carries over seamlessly to combinatorial semi-bandits, and yields the following guarantee:

\begin{lemma}[\citet{ICexploration-ec15}]\label{lem:bic-ec15}
Assume Property \PropHE holds with constants $\myNp, \tau_\calP, \rho_\calP$ and $\kappa=\kappa(\myNp)<\infty$. Then \Cref{alg:cyclic-exploration} with parameters $N\ge \myNp$ and $L$ satisfying \eqref{eq:L-bound} is BIC:
\begin{equation}\label{eq:L-bound}
    L \geq 1 + \frac{\mu^0_{\max} - \mu^0_{\min}}{\tau_{\calP} \cdot \rho_{\calP}},
    %L \geq 1 + \frac{\E\sbr{\mu(A)} - \mu^0_{\min}}{\tau_{\calP} \cdot \rho_{\calP}}.
\end{equation}
where 
    $\mu_{\max}^0=\max_{A\in\calA} \E[\mu(A)]$
and
    $\mu_{\min}^0=\min_{A\in\calA} \E[\mu(A)]$.
%where $\mu_{\max}^0=\max_{A\in\calA}\sum_{\ell\in A}\theta^0_\ell$ and $\mu_{\min}^0=\min_{A\in\calA}\sum_{\ell\in A}\theta^0_\ell$.
\end{lemma}

\iffalse
\begin{proof}
In this proof, the principal only needs to see the total reward of atoms in the chosen arm. %We prove the algorithm is BIC for each phase $i$. We wrap up the samples from previous phases as $S_i=\{S_1^{n}, \cdots, S_{i-1}^{n}\}$.
Then our proof follows the proof in Section 5.1 in \citet{mansour2015bayesian} by endowing the arm set and ordering as the arm sequence in our case.
\iffalse
and wrapping up their history samples as $S_i$ for each phase $i$ as well.
Then we get phase $i$ is BIC if $L$ satisfies
\begin{equation}
    L \geq 2 + \frac{\mu_{\max}^0 - \mu_{V_i}^0}{\tau_{\calP} \cdot \rho_{\calP}} \quad \text{where } \mu_{\max}^0 = \E_{\calP}[\max_{A \in \calA} \mu_A]
\end{equation}
Since $\mu_{\max}^0- \mu_{V_i}^0\ge \mu_{\max}^0- \mu_{\min}^0\ge m(\theta_d-\theta_0)$, then $L \geq 2 + \frac{m(\theta_d-\theta_0)}{\tau_{\calP} \cdot \rho_{\calP}}$ suffices to make all phases BIC.
%\swcomment{similar but not the same? }\vhcomment{after wrapping up history samples, remaining part of the proof should follow the same except for the notations.}
\fi
\end{proof}
\fi

\begin{proof}[Proof of \Cref{thm-p2-bic}]
It remains to interpret and simplify the quantities in \Cref{lem:bic-ec15}. 
%Assume \PropHE holds with constants $\myNp, \tau_\calP, \rho_\calP$ and $\kappa=\kappa(\myNp)$. 
According to \Cref{lem:bic-ec15}, \Cref{alg:cyclic-exploration} is BIC with parameters $N\ge \myNp$ and $L$ satisfying \Cref{eq:L-bound}. It suffices to take $N=\myNp$. Since $\theta^0_\ell\in[0, 1]$ for any $\ell\in[d]$, we have $0\le \mu^0_{\min}\le \mu^0_{\max}\le d$ and $0\le \mu^0_{\max}-\mu^0_{\min}\le d$. Additionally, $\tau_\calP, \rho_\calP\in (0,1)$. So \[1+\frac{\mu^0_{\max}-\mu^0_{\min}}{\tau_{\calP} \cdot \rho_{\calP}}\le 1 + \frac{d}{\tau_{\calP} \cdot \rho_{\calP}}\le \frac{1+d}{\tau_{\calP} \cdot \rho_{\calP}}.\]
And thus it suffices to take $L = \frac{1+d}{\tau_{\calP} \cdot \rho_{\calP}}$. Then we have the total number of rounds $T_0=\kappa\cdot N\cdot L = \kappa\cdot \myNp\cdot (1+d) \,/\, (\tau_{\calP} \cdot \rho_{\calP})$.
\end{proof}
\subsection{Restricted family of arms: Proof of \Cref{thm-P2-beta}}\label{subsec:initial-fixedsize}

Firstly, according to Assumption \eqref{eq:cyclic-assn-1} and the reward support $\Theta\subset[0,1]$, we observe that:
\begin{align}
&\text{the prior/posterior-best arm contains the $m$ prior/posterior-best atoms;}\label{obs-best}\\
&\text{the second prior/posterior-best arm contains the $m-1$ prior/posterior-best atoms.}\label{obs-second}
\end{align}

Then, according to our choice of $\myNp$ (\ref{eq:number-sequence-beta-frac}), we will prove that $\kappa=\kappa(\myNp)$ is finite (i.e. our arm sequence will contain all atoms at least once). Actually, we will prove $\kappa=\Cel{d/m}$ by proving the following \Cref{cla:finite-kappa}.
\begin{claim}\label{cla:finite-kappa}
Assume Beta-Bernoulli priors (\ref{eq:cyclic-assn-2}-\ref{eq:cyclic-assn-3}), all arms have a fixed size (\ref{eq:cyclic-assn-1}), and $\myNp$ satisfies (\ref{eq:number-sequence-beta-frac}). Then the arm sequence $V_1, V_2, \cdots $, where $V_i=V_i^{\myNp}$, have the following properties:
\begin{align}
    V_i= \cbr{ (i-1)m+\ell : \ell\in [m] }, \quad& i\in[\Cel{d/m}-1];\label{eq:inductive-1}\\
    V_i\supset \{(i-1)m+\ell: \ell\in [m], (i-1)m+\ell\le d \}, \quad& i=\Cel{d/m}.\label{eq:inductive-2}
\end{align}
And thus $\kappa(\myNp)=\Cel{d/m}$.
\iffalse
\begin{equation}
    V_i\supset \cbr{ im+\ell : \ell\in [m],  im+\ell\le d }, \quad i\in[\Cel{d/m}];
    %V_i\supset \{im+\ell: \ell\in [m], im+\ell\le d \}, \quad& i=\Cel{d/m};
\end{equation}
\begin{equation}\label{eq:inductive-3}
    \Pr \sbr{X_i^N\ge \tau_\calP\mid \myEv_{i-1}}=1.
\end{equation}
\fi
\end{claim}

\begin{proof}
We will prove by induction on phase $i$.
For phase $i=1$, $V_1$ is the prior-best arm. According to the observation \eqref{obs-best}, $V_1$ contains the largest $m$ prior-best atoms, which is $[m]$.
\iffalse
Let $A_1$ is the second prior-best arm. According to observation \ref{obs-second},
\[A_1=[m-1]\cup\{m+1\}.\]
Recall the definition of $X_i^N$ and $\tau_\calP$, given empty set $H_0^N$ and its corresponding full event $\myEv_0$, \[X_1^N=\E\sbr{\mu(V_1)-\mu(A_1)}=\sum_{\ell\in V_1} \nu_\ell(0)-\sum_{\ell'\in A_1}\nu_{\ell'}(0)=\nu_m(0)-\nu_{m+1}(0)\ge \tau_\calP.\]
Then we have $\Pr\sbr{X_1^N\ge \tau_\calP\mid \myEv_0}=1$.
\fi

Suppose the induction hypothesis is true for all phases up to some phase $i\in [\Cel{d/m}-1]$. Denote $B_i$ as a subset of atoms having been contained at least once in the first $i$ arms and $\Bar{B_i}$ as the complement subset of atoms. Then
\[B_i=\bigcup_{j\in [i]} V_j=[im] \text{ and } \Bar{B_i}=[d]-[im].\]
Recall the definition of $Z_\ell^\myNp$ and $\nu_\ell(n)$, we have for each atom $\ell\in B_i$ and $\ell'\in \Bar{B_i}$:
\[Z_\ell^\myNp=\myNp \text{ and } \nu_\ell(Z_\ell^\myNp)=\nu_\ell(\myNp)=\alpha_\ell/(\alpha_\ell+\beta_\ell+\myNp);\] %while for each atom $\ell'\in \Bar{B_i}$,
\[Z_{\ell'}^\myNp=0 \text{ and } \nu_{\ell'}(Z_{\ell'}^\myNp)=\nu_{\ell'}(0)=\alpha_{\ell'}/(\alpha_{\ell'}+\beta_{\ell'})=\theta_{\ell'}^0%\ge\theta_d^0=\alpha_{d}/(\alpha_{d}+\beta_{d})
.\]
Since $\theta_1^0\ge \LDOTS \ge \theta_d^0$, we have:
\begin{align}\label{eq:nu-1}
    \nu_{\ell'}(Z_{\ell'}^\myNp) \text{ decreases in } \ell'\in \Bar{B_i}.
\end{align}
By definition of $\myNp$ and $\theta_{\ell'}^0\ge \theta_d^0$, we have:
\[\alpha_\ell/(\alpha_\ell+\beta_\ell+\myNp)<\alpha_{d}/(\alpha_{d}+\beta_{d})\le \alpha_{\ell'}/(\alpha_{\ell'}+\beta_{\ell'})=\theta_{\ell'}^0.\]
Thus: %we have the posterior mean of any atom $\ell'\in\Bar{B_i}$ is greater than the one of any atom $\ell\in B_i$:
\begin{align}\label{eq:nu-2}
    \nu_\ell(Z_\ell^\myNp)<\nu_{\ell'}(Z_{\ell'}^\myNp), \forall \ell\in B_i, \ell'\in \Bar{B_i}.
\end{align}
Combining \eqref{eq:nu-1}-\eqref{eq:nu-2} and according to the observation \eqref{obs-best}, we have $V_{i+1}$ for phase $i+1$. %\ref{eq:inductive-1}-\ref{eq:inductive-2}.
%and the lower bound of $X_{i+1}$ conditional on $\myEv_{i}$.
If $i\in[\Cel{d/m}-2]$, we have $|\Bar{B_i}|=d-im\ge m+1$. Thus $im+1<\cdots<(i+1)m\le d-1$ and $V_{i+1}=\{im+\ell: \ell \in [m]\}$. Otherwise for $i=\Cel{d/m}-1$, we have $1\le |\Bar{B_i}|\le m$. Thus $V_{i+1}\supset \Bar{B_i}=\{im+\ell: \ell \in [m]\text{ and } im+\ell\le d\}$. Thus, the induction hypothesis is true for phase $i+1$ and we complete the induction proof. And since $V_1\LDOTS V_{\Cel{d/m}}$ contain all atoms, we have $\kappa(\myNp)=\Cel{d/m}$.  %\vhcomment{definition of arm sequence and $\kappa$ in the body did not declare the case of non-existence, might causing some ambiguity there.}
\end{proof}

%\paragraph{Claim for \Cref{thm-P2-beta} and \Cref{thm-P2-gen}}
%Before our proof of \Cref{thm-P2-beta} and \Cref{thm-P2-gen},
Secondly, we define an event and give a lower bound of the probability of this event. Given any $\myNp, N\in \N$ ($\myNp\le N$) and $H_i^N, \forall i\in[\kappa]$, define an event $\myEv_i$ for each $i\in[\kappa]$ saying that the first $\myNp$ reward samples of each atom in $\bigcup_{j\in[i]} V_j$ are $0$. Formally,
\begin{equation}\label{eq:myEv}
    \myEv_i=\left\{r_\ell\supt=0, \forall \ell\in \bigcup_{j\in[i]} V_j, t\in[\myNp]\right\}, \forall i\in[\kappa].
\end{equation}
where we abuse the notation of $r_\ell\supt$ as the $t$-th round that atom $\ell$ is being contained. Since $H_0^N$ is an empty data set, we define $\myEv_0$ is a full event, which gives no information wherever it applies. %\vhcomment{how to define $\myEv_0$? a full event?}

Then, according to our choice of $\rho_\calP$(\ref{eq:positive-prob}), we will lower bound the probability of the event defined above in the following claim.
\begin{claim}\label{cla:positive-prob}
Assume Beta-Bernoulli priors (\ref{eq:cyclic-assn-2}-\ref{eq:cyclic-assn-3}) and $\rho_\calP$ satisfies (\ref{eq:positive-prob}). Then for any given $\myNp\le N$, with the definition of $\myEv_i, \forall i\in[\kappa]$ \eqref{eq:myEv}, we have:
\begin{align}
    \Pr\sbr{\myEv_i}\ge \rho_\calP, \forall i\in [\kappa].
\end{align}
\end{claim}
\begin{proof}
Firstly, by the prior and reward independence among each atoms:
\begin{align*}
\Pr\sbr{\myEv_i}
    &=\Pr\sbr{r_\ell\supt=0, \forall \ell\in \bigcup_{j\in[i]} V_j, t\in[\myNp]}\\
    &=\prod_{\ell\in \bigcup_{j\in[i]} V_j} \Pr\sbr{r_\ell\supt=0, \forall t\in[\myNp]}
\end{align*}
Secondly, for each given atom $\ell$, by the independence among realized rewards conditioned on the mean reward drawn from the prior and iteratively using Harris Inequality:
\begin{align*}
\Pr\sbr{r_\ell\supt=0, \forall t\in[\myNp]}
    &=\E_{\theta_\ell}\sbr{\Pr_{r_\ell\supt}\sbr{r_\ell\supt=0, \forall t\in [\myNp]}\mid \theta_\ell}\\
    &=\E_{\theta_\ell}\sbr{\prod_{t\in [\myNp]}\Pr_{r_\ell\supt}\sbr{r_\ell\supt=0\mid \theta_\ell}}\tag{conditional independence}\\
    &=\E_{\theta_\ell}\sbr{\prod_{t\in [\myNp]} (1-\theta_\ell)} \tag{Bernoulli rewards \eqref{eq:cyclic-assn-3}}\\
    &\ge \prod_{t\in[\myNp]} \E_{\theta_\ell}\sbr{ (1-\theta_\ell)}\tag{Harris inequality}\\
    &=(1-\theta_\ell^0)^\myNp%\tag{Beta priors \eqref{eq:cyclic-assn-2}}\\
    %&=\rho_\calP \tag{\Cref{eq:positive-prob}}
\end{align*}
\iffalse
\begin{align}
\Pr[r_\ell\supt=0, \forall t\in[\myNp]]
    &=\E_{\theta_\ell}[\Pr_{r_\ell\supt}[r_\ell\supt=0, \forall t\in [\myNp]]\mid \theta_\ell]\\
    &=\E_{\theta_\ell}[\prod_{t\in [\myNp]}\Pr_{r_\ell\supt}[r_\ell\supt=0\mid \theta_\ell]]\tag{conditional independence}\\
    &=\E_{\theta_\ell}[(1-\theta_\ell)^\myNp] \tag{Bernoulli reward}\\
    &\ge (\E_{\theta_\ell}[(1-\theta_\ell)])^\myNp \tag{Harris inequality}\\
    &=(1-\theta_\ell^0)^\myNp\tag{Beta prior}
\end{align}
\fi
Combining both and recall that $\theta_1^0\ge \cdots \ge \theta_d^0$:
\begin{align*}
\Pr\sbr{\myEv_i}
    \ge \prod_{\ell\in \bigcup_{j\in[i]} V_j} (1-\theta_\ell^0)^\myNp
    \ge \prod_{\ell\in [d]} (1-\theta_\ell^0)^\myNp
    \ge (1-\theta_1^0)^{d\myNp}=\rho_\calP.
\end{align*}
\end{proof}

Note that this second part analysis does not rely on \eqref{eq:cyclic-assn-1} and we will reuse that part for the proof of general feasible arm set case in \Cref{subsec:initial-general}.

Thirdly, according to our choice of $\tau_\calP$ (\ref{eq:min-gap}), we will prove the following claim, which says the expectation $X_i^N$ conditioned on the event $\myEv_i$ almost surely $\ge \tau_\calP$ for any phase $i\in[\kappa]$ and any $N\ge \myNp$.
\begin{claim}\label{cla:positive-gap}
Assume Beta-Bernoulli priors (\ref{eq:cyclic-assn-2}-\ref{eq:cyclic-assn-3}), all arms have a fixed size (\ref{eq:cyclic-assn-1}) and $\tau_\calP$ satisfies (\ref{eq:min-gap}). Then for any given $\myNp$, we have:
\begin{equation}\label{eq:positve-gap}
    \Pr\sbr{X_i^N\ge \tau_\calP \mid \myEv_{i-1}}=1, \forall i\in[\kappa], N\ge \myNp.
\end{equation}
\end{claim}

\begin{proof}
For each phase $i\in[\kappa]$, let $A_i$ is the second prior/posterior-best arm conditioned on $H_{i-1}$ and $\myEv_{i-1}$. According to observation \eqref{obs-best}-\eqref{obs-second} and the definition of $\tau_\calP$:
%the lower bound of $X_{i+1}$ conditional on $\myEv_{i}$:
%$V_{i+1}$ contains the smallest $m$ atoms in $\Bar{B_i}$ if $|\Bar{B_i}|\ge m$;
\begin{align*}
\min_{\text{arm }A\neq V_{i}} \E\sbr{\mu(V_{i})-\mu(A)\mid H_{i-1}^N, \myEv_{i-1}}
    &=\E\sbr{\mu(V_{i})-\mu(A_{i})\mid H_{i-1}^N, \myEv_{i-1}}\\
    &=\sum_{\ell\in V_{i}}\nu_\ell(Z_\ell^\myNp)-\sum_{\ell'\in A_{i}}\nu_{\ell'}(Z_{\ell'}^\myNp)\\
    &=\min_{\ell\in V_{i}}\nu_\ell(Z_\ell^\myNp)-\min_{\ell'\in A_{i}}\nu_\ell(Z_{\ell'}^\myNp)\\
    &\ge \min_{\ell, \ell'\in [d], n, n'\in \{0, \myNp\}} \abs{\nu_\ell(n)-\nu_{\ell'}(n')}\\
    &=\tau_\calP.
\end{align*}
Thus we have \Cref{eq:positve-gap}.
%\[\Pr \sbr{X_i^N\ge \tau_\calP\mid \myEv_{i-1}}=1, \quad i\in[\kappa].\]% \label{eq:inductive-3}
%Combine this with \Cref{cla:positive-prob} and use observation \ref{eq:decompose-prob}, then we finish the proof for \Cref{thm-P2-beta}.
\end{proof}

At last, combining the claims above,
%\Cref{cla:finite-kappa}, \Cref{cla:positive-prob} and  \Cref{cla:positive-gap},
we have for each $i\in [\kappa]$:
\begin{align*}%\label{eq:decompose-prob}
\Pr\sbr{X_i^N\ge \tau_\calP}
    &\ge \Pr\sbr{\myEv_{i-1}} \cdot \Pr\sbr{X_i^N \ge \tau_\calP\mid \myEv_{i-1}}\\
    &\ge \rho_\calP \cdot \Pr\sbr{X_i^N\ge \tau_\calP \mid \myEv_{i-1}}\tag{\Cref{cla:positive-gap}}\\
    &=\rho_\calP\cdot 1\tag{\Cref{cla:positive-prob}}\\
    &=\rho_\calP,
\end{align*}
which implies $\PropHE$.
%and we finish the proof.

%\input{appendix/semi-initial-time-bound}

\subsection{Arbitrary family of arms: Proof of \Cref{thm-P2-gen}}\label{subsec:initial-general}

We prove \Cref{thm-P2-gen}, reusing much of the proof of \Cref{thm-P2-beta}. While the parameters in \Cref{thm-P2-gen} give a weaker bound on the number of rounds, the proof becomes more intuitive.

%\paragraph{Proof of \Cref{thm-P2-gen}}
Firstly, we prove $\kappa$ is finite in this following claim.
\begin{claim}\label{cla:finite-kappa-gen}
Assume Beta-Bernoulli priors (\ref{eq:cyclic-assn-2}-\ref{eq:cyclic-assn-3}) and $\myNp$ satisfies (\ref{eq:number-sequence-beta-frac-gen}). Then $\kappa(\myNp)\le d$.
\end{claim}
\begin{proof}
Denote the explored atom set up to phase $i$ as $B_i=\bigcup_{j\in[i]} V_j$ and the unexplored atom set as $\Bar{B_i}=[d]-B_i$. %$\forall i\in[\kappa-1]$.
Denote $B_0=\emptyset$ and $\Bar{B_0}=[d]$.
Fixed $i\ge 0$.
Recall the definition of $Z_\ell^\myNp$ and $\nu_\ell(n)$, we have for each atom $\ell\in B_i$ and $\ell'\in \Bar{B_i}$:
\[Z_\ell^\myNp=\myNp \text{ and } \nu_\ell(Z_\ell^\myNp)=\nu_\ell(\myNp)=\alpha_\ell/(\alpha_\ell+\beta_\ell+\myNp);\] %while for each atom $\ell'\in \Bar{B_i}$,
\[Z_{\ell'}^\myNp=0 \text{ and } \nu_{\ell'}(Z_{\ell'}^\myNp)=\nu_{\ell'}(0)=\alpha_{\ell'}/(\alpha_{\ell'}+\beta_{\ell'})%=\theta_{\ell'}^0%\ge\theta_d^0=\alpha_{d}/(\alpha_{d}+\beta_{d})
.\]
\iffalse
Since $\theta_1^0\ge \LDOTS \ge \theta_d^0$, we have:
\begin{align}\label{eq:nu-1}
    \nu_{\ell'}(Z_{\ell'}^\myNp) \text{ decreases in } \ell'\in \Bar{B_i}.
\end{align}
\fi
By definition of $\myNp$ and $\theta_{\ell'}^0\ge \theta_d^0$, we have:
\[\nu_{\ell'}(Z_{\ell'}^\myNp)=\alpha_\ell/(\alpha_\ell+\beta_\ell+\myNp)<\frac{1}{d}\alpha_{d}/(\alpha_{d}+\beta_{d})\le \frac{1}{d} \alpha_{\ell'}/(\alpha_{\ell'}+\beta_{\ell'})=\frac{1}{d}\nu_{\ell'}(Z_{\ell'}^\myNp).\]
Thus: %we have the posterior mean of any atom $\ell'\in\Bar{B_i}$ is greater than the one of any atom $\ell\in B_i$:
\begin{align}
% \label{eq:nu-2}
    \sum_{\ell\in B_i} \nu_\ell(Z_\ell^\myNp)<d\cdot \frac{1}{d}\nu_{\ell'}(Z_{\ell'}^\myNp) =\nu_{\ell'}(Z_{\ell'}^\myNp), \forall \ell'\in \Bar{B_i}.
\end{align}
Then according to the definition of $V_{i+1}$, we know $V_{i+1}$ contains at least one atom $\ell'\in \Bar{B_i}$.
%Additionally, before phase $1$, there is no atom being explorer. So $V_1$ contains at least one unexplored atom.
So the number of uncovered atoms, i.e. $|\Bar{B_i}|$, decreases at least $1$ after each phase. Thus it takes at most $d$ phases to cover all atoms, which implies $\kappa\le d$.
\end{proof}

Secondly, we reuse the definition of $\myEv_i$ (\ref{eq:myEv}) and \Cref{cla:positive-prob} to give a lower bound of $\Pr\sbr{\myEv_i}$, since this part in \Cref{subsec:initial-fixedsize} don't rely on Assumption \eqref{eq:cyclic-assn-1}.

Thirdly, according to our definition of $\tau_\calP$ (\ref{eq:min-gap-gen}), we have the following claim similar to \Cref{cla:positive-gap}.
\begin{claim}\label{cla:positive-gap-gen}
Assume Beta-Bernoulli priors (\ref{eq:cyclic-assn-2}-\ref{eq:cyclic-assn-3}) and $\tau_\calP$ satisfies (\ref{eq:min-gap-gen}). Then for any given $\myNp$, we have:
\begin{equation}\label{eq:positve-gap-gen}
    \Pr\sbr{X_i^N\ge \tau_\calP \mid \myEv_{i-1}}=1, \forall i\in[\kappa], N\ge \myNp.
\end{equation}
\end{claim}
\begin{proof}
For each phase $i\in [\kappa]$, we have:
\begin{align*}
\min_{\text{arm }A\neq V_{i}} \E\sbr{\mu(V_{i})-\mu(A)\mid H_{i-1}^N, \myEv_{i-1}}
    &\ge \min_{A\neq A'\in \calA} \E\sbr{\mu(A)-\mu(A')\mid H_{i-1}^N, \myEv_{i-1}}\\
    &=\min_{A\neq A'\in \calA, n, n'\in\{0, \myNp\}^d} \abs{\sum_{\ell\in A}\nu_\ell(n)-\sum_{\ell'\in A'}\nu_\ell(n')}\\
    &=\tau_\calP.
\end{align*}
Then we have \Cref{eq:positve-gap-gen}.
\end{proof}

At last, similar to the last step in \Cref{subsec:initial-fixedsize}, combining \Cref{cla:finite-kappa-gen}, \Cref{cla:positive-prob} and  \Cref{cla:positive-gap-gen}, we have for each $i\in [\kappa]$:
\iffalse
\begin{equation}\label{eq:decompose-prob}
\Pr\sbr{X_i^N> \tau_\calP}
    \ge \Pr\sbr{\myEv_{i-1}} \cdot \Pr\sbr{X_i^N > \tau_\calP\mid \myEv_{i-1}}
    \ge \rho_\calP \cdot \Pr\sbr{X_i^N>\tau_\calP \mid \myEv_{i-1}}
    =\rho_\calP\cdot 1
    =\rho_\calP,
\end{equation}
\fi
\begin{align*}%\label{eq:decompose-prob}
\Pr\sbr{X_i^N\ge \tau_\calP}
    &\ge \Pr\sbr{\myEv_{i-1}} \cdot \Pr\sbr{X_i^N \ge \tau_\calP\mid \myEv_{i-1}}\\
    &\ge \rho_\calP \cdot \Pr\sbr{X_i^N\ge \tau_\calP \mid \myEv_{i-1}}\tag{\Cref{cla:positive-gap-gen}}\\
    &=\rho_\calP\cdot 1\tag{\Cref{cla:positive-prob}}\\
    &=\rho_\calP,
\end{align*}

which implies $\PropHE$.

\subsection{Motivation for assumption \eqref{eq:cor-p2-gen-2}}
\label{sec:mild-assumption}

Let us provide some motivation for why \eqref{eq:cor-p2-gen-2} is a mild assumption.

Fix a vector $n\in \N^d$ and define
\begin{align}\label{eq:app-HE-tauP}
\tau_\calP(n)
    &=\min_{A\neq A'\in \calA}
        \abs{\sum_{\ell\in A}\nu_\ell(n_\ell)-\sum_{\ell'\in A'}\nu_{\ell'}(n_{\ell'})}.
\end{align}

Our intuition is as follows: $\tau_\calP(n)$ is defined as the smallest difference between $e^{O(d)}$ numbers in $[-d,d]$, so typical situation should be that $\tau_\calP$ is on the order of $e^{-O(d)}$, whereas our assumption only requires it to be larger than $e^{-\Omega(d^2)}$.

We make this intuition precise, in a sense defined below. We argue that $\tau_\calP(n)$ is ``not too small" for ``all but a few" problem instances. More formally, we define a distribution over problem instances such that
    $\tau_\calP(n) \geq \Omega(c_2^{-d^2})$
with very high probability. For instance, we can make it hold with probability at least $1-\delta/2^d$ for some small $\delta>0$.

(However, we do not construct one distribution that works for all relevant vectors $n$ at once, although we suspect that our technique, based on Esseen inequality, might be extended there.)

So, let us construct the desired distribution over problem instances. We fix $d$ and the set of feasible arms, and we only vary the per-atom priors. Recall that the prior $\calP_\ell$ for a given atom $\ell\in[d]$ is specified by a pair of numbers, $(\alpha_\ell,\,\beta_\ell)$. Further, recall that
    $\nu_\ell(n_\ell) = \alpha_\ell\,/\,(\alpha_\ell+\beta_\ell+n_\ell)$, $n_\ell\in\N$.
We require that $\nu_\ell(n_\ell)$ is distributed uniformly on some interval.

%In the following lemma, we provide an example of the lower bound on $\tau_{\calP}$ in \Cref{cor-p2-gen}. For any given vector $n=\{0, n_{\calP}\}^d$, we can define a prior-dependent constant $\tau_{\calP}$ according to $n$. Informally, when each $\nu_\ell$ is uniformly distributed in $[a,b]$, then $\tau_{\calP}(n)$ is larger than $\Omega(c_2^{-d})$ with probability almost $1$.

\begin{lemma} \label{lm:motivation}
Fix vector $n\in\N^d$. Suppose for each atom $\ell\in[d]$, the pair $(\alpha_\ell, \beta_\ell)$ is drawn independently from some distribution such that $\nu_\ell(n_\ell)$ is uniformly distributed in some interval $[a_\ell, b_\ell]$. Fix $\delta\in(0,1)$.
Then it holds that
\begin{equation}
    \Pr \sbr{\tau_{\calP}(n) < \frac{\delta}{2\cdot8^d}} \leq \frac{\delta}{2^d}.
\end{equation}
\end{lemma}

\begin{remark}
One way to ensure that $\nu_\ell(n_\ell)$ is uniformly distributed is as follows. Fix atom $\ell$, and parameters $\beta_\ell$ and $n_\ell$. Let $\nu_\ell = \nu_\ell (n_\ell)$. Note that
\begin{equation}
    \alpha_\ell = \Phi(\nu_\ell) := \frac{\nu_\ell \cdot ( \beta_\ell + n_\ell)}{1 - \nu_\ell}.
\end{equation}
Now, just let $\alpha_\ell$ be distributed as $\Phi(Y)$, where $Y$ is uniform on $[a_\ell,b_\ell]$ interval, where $0 \leq a_\ell \leq \alpha_\ell \leq b_\ell \leq 1$. Observe that by change-of-variable, when $\alpha_\ell$ is distributed with $\Phi(Y)$, then $\nu_\ell$ is distributed uniformly on $[a_\ell,b_\ell]$.
\end{remark}

To prove \Cref{lm:motivation}, we invoke the following tool from \emph{anti-concentration}.

\begin{theorem}[Esseen inequality]\label{thm:esseen}
Let $Y$ be a random variable. Consider its \emph{characteristic function},
\[ \psi_Y(\lambda) = \E\sbr{e^{i\lambda Y}}, \quad \lambda\in \R .\]
Then for any $x>0$ it holds that
\begin{align*}
Q_Y(x)
    \coloneqq \sup_{y \in \R} \Pr\sbr{\abs{Y - y} \leq x}
    \leq x \int_{-2\pi/x}^{2\pi/x} \abs{\psi_Y(\lambda)}\; d\lambda.
\end{align*}
\end{theorem}

\begin{proof}[Proof of Lemma~\ref{lm:motivation}]
Let $\nu_\ell = \nu_\ell(n_\ell)$ for each atom $\ell\in[d]$. Focus on
\[ X := \sum_{\ell\in A}\nu_\ell-\sum_{\ell'\in A'}\nu_{\ell'}.\]
We treat $X$ as a random variable, under the distribution over the $(\alpha_\ell, \beta_\ell)$ pairs. Without loss of generality, from here on assume that arms $A$ and $A'$ are disjoint subsets of atoms.

We will use Theorem~\ref{thm:esseen} to prove that $Q_X(x)\leq 2x$ for any $x>0$.
%$Q_X(x)\leq c x \log(x) $ for any $x>0$ and some absolute constant $c$.

Since $X$ is a sum of independent random variables $ \pm \nu_\ell$, $\ell\in A\cup A'$, the characteristic function of $X$ is the product of the respective characteristic functions
\[ \psi_X(\lambda) = \prod_{\ell\in A\cup A'} \psi_\ell(\lambda), \]
where
    $\psi_\ell(\lambda) = \psi_{\nu_\ell}(\lambda)$
for $\ell\in A$, and
    $\psi_\ell(\lambda) = \psi_{-\nu_\ell}(\lambda)$
for $\ell\in A'$.

From here on, fix some atom $\ell\in A$.  Since $|\psi_Y(\lambda)|\leq 1$ for any random variable $Y$, it follows that
\[ |\psi_X(\lambda)| \leq |\psi_{\nu_\ell}(\lambda)|. \]

For the rest of the proof we focus on the characteristic function for $\nu_\ell$,
     $\psi(\cdot) := \psi_{\nu_\ell}(\cdot)$.

Recall that $\nu_\ell$ is distributed uniformly on some interval $[a,b] = [a_\ell,b_\ell]$. A known fact about characteristic function of uniform distribution is that
 \[ \psi(\lambda) = \frac{e^{i\lambda b} - e^{i\lambda a}}{i \lambda (b-a)}.\]

The rest of the proof is a simple but somewhat tedious integration.
By Esseen inequality,

\begin{align}
    Q_X(x) &\leq x \int_{-2\pi/x}^{2\pi/x} \abs{\psi(\lambda)} d\lambda\\
    &= x \int_{-2\pi/x}^{2\pi/x} \abs{\frac{e^{i \lambda b} - e^{i \lambda a}}{i \lambda (b-a) }} d \lambda\\
    &= x \int_{-2\pi/x}^{2\pi/x} \frac{\sqrt{2 - 2\cos(\lambda b - \lambda a)}}{\abs{\lambda (b-a)}} d\lambda
\end{align}
Let $u = \lambda(b-a)$. Then by substitution we have:
\begin{align}
   Q_X(x) &\leq x \int_{-2\pi(b-a)/x}^{2\pi(b-a)/x} \sqrt{\frac{2 - 2\cos(u)}{u^2}} du\\
   &= x \int_{-2\pi(b-a)/x}^{2\pi(b-a)/x} \sqrt{\frac{2 - 2\sum_{n=0}^\infty \frac{(-1)^n u^{2n}}{(2n)!} }{u^2}} du \tag{by Maclaurin series of $\cos(u)$}\\
   &= x \int_{-2\pi(b-a)/x}^{2\pi(b-a)/x} \sqrt{2} \sqrt{\frac{1}{2} - \frac{u^2}{4!} + \cdots} du
\end{align}
When we have $\abs{u} \leq 1$, the terms in the integrand are decreasing. Hence, the entire integrand can be upper bounded by $1$. Otherwise, when $\abs{u} \geq 1$, we can upper bound the integrand by $\nicefrac{2}{\abs{u}}$. Hence, the concentration function $Q_X(t)$ is upper bounded by:
\begin{align*}
    Q_X(x) &\leq x \int_{-2\pi(b-a)/x}^{-1} \frac{2}{\abs{u}} du +  \int_{-1}^1 1 du + \int_{1}^{2\pi(b-a)/x} \frac{2}{\abs{u}} du\\
    &= x \int_{1}^{2\pi (b-a)/x} \frac{-2}{\abs{v}} dv + \int_{-1}^1 1 du + \int_{1}^{2\pi(b-a)/x} \frac{2}{\abs{u}} du \tag{substitution $v = -u$}\\
    &= x \rbr{-2 \log(v) \Bigg |_{1}^{2\pi(b-a)/x} + 2 + 2 \log(u) \Bigg |_{1}^{2\pi(b-a)/x}}\\
    &= 2x
    % \\
    % &\leq c x \log\rbr{\frac{1}{x}}
\end{align*}
% where $c$ is a small absolute constant.
For $x = \nicefrac{\delta}{2\cdot8^d}$, we have $Q_X(\nicefrac{\delta}{2\cdot8^d}) \leq \nicefrac{\delta}{8^d}$.
Observe that since $A$ and $A'$ are subsets of atoms, there are at most $2^d$ possible choices for each arm $A$ and $A'$. Hence, for a fixed vector $n$, there are $4^d$ possible values of $X$. By union bound, we have $\Pr[\tau_{\calP}(n) \leq \nicefrac{\delta}{2\cdot 8^d}] \leq \nicefrac{\delta}{2^d}$.
%Hence, for $x = \frac{1}{32^d}$, we have $Q_X(\nicefrac{1}{32^d}) \leq \nicefrac{c}{32^d}$. Taking union bound over all possible values of $n_\ell$, we have $\Pr[X \leq \nicefrac{1}{32^d}] \leq \frac{c}{2^d}$.
\end{proof}
 
%\section{Auxiliary theorems}
\section{Probabilistic tools from prior work}

In this appendix, we spell out some probabilistic tools from prior work that we rely on.

\subsection*{Bayesian Chernoff Bound}

We use an easy corollary of the Bayesian Chernoff Bound from \citet{Selke-PoIE-ec21}.

% AS: the proofs were needed for bandit feedback only! DO NOT ERASE THEM.

%\begin{lemma}
%\label{lem:high-prob-posterior-bound}
%Suppose there exists an estimation $\hat{\theta} = \theta(\gamma)$ for a parameter $\theta$ given an observation signal $\gamma$ which satisfies a concentration inequality $\Pr[| \hat{\theta} - \theta| \geq \epsilon] \leq \delta$. Assume we start with a prior $\calP$ over $\theta$ before observing $\gamma$. Then if $\theta$ is in fact chosen according to $\calP$ and $\tilde{\theta}$ is a sample from the posterior distribution for $\theta$ conditional on $\gamma$, we have $\Pr[| \tilde{\theta} - \theta | \geq 2\epsilon] \leq 2\delta$.
%\end{lemma}
%
%\begin{proof}
%By assumption, $\Pr[\abs{\hat{\theta} - \theta} > \epsilon] \leq \delta$. Also, since $(\theta, \hat{\theta})$ and $(\tilde{\theta}, \hat{\theta})$ are identically distributed; choosing $\tilde{\theta}$ amounts to resampling from the joint distribution of $(\theta, \hat{\theta})$. Therefore, $\Pr[\abs{\hat{\theta} - \tilde{\theta}} > \epsilon] \leq \delta$. Combining these two inequalities completes our proof by triangle inequality.
%\end{proof}

%\begin{lemma} (added for compatibility with the submitted body of the paper, please ignore)
%\end{lemma}

%\ascomment{needs to be Lemma E.2 for compatibility with the submitted body.} \vhcomment{D.2?}

\begin{lemma}[\citet{Selke-PoIE-ec21}]%[Bayesian Chernoff Bound]
Fix round $t$ and parameters $\epsilon, \tau > 0$. Suppose algorithm's history $\calF_t$ almost surely contains at least $\epsilon^{-2}$ samples of each atom. Let $\tilde{\theta}$ be a posterior sample for the mean reward $\theta$, \ie $\tilde{\theta}$ is an independent sample from the posterior distribution on $\theta$ given $\calF_t$. Then for some universal absolute constant $C$, we have
\begin{align}
    &\Pr \sbr{ \nrm{\tilde{\theta} - \theta} \geq \tau \epsilon } \leq C\cdot e^{-\tau^2/C},\\
    &\Pr \sbr{ \nrm{\E[\theta | \calF_t] - \theta} \geq \tau \epsilon } \leq C\cdot e^{-\tau^2/C}.
\end{align}
\label{lem:bayesian-chernoff}
\end{lemma}

\begin{proof}
\citet{Selke-PoIE-ec21} contains this result for $d=1$ atoms. Here, we apply the result from \citet{Selke-PoIE-ec21} to each atom separately, using the fact that the Bayesian update is independent across atoms.
\end{proof}

\subsection*{Harris Inequality}

We invoke Harris Inequality about correlated random variables.

%\vhcomment{needs to be Theorem D.3 for compatibility with the submitted body.}
\begin{theorem}[\citet{harris_1960}]
Let $f,g: \mathbb{R}^n \rightarrow \mathbb{R}$ be nondecreasing functions. Let $X_1, \cdots, X_n$ be independent real-valued random variables and define the random vector $X = (X_1, \cdots, X_n)$ taking values in $\mathbb{R}^n$. Then
\begin{equation*}
    \E\sbr{f(X) \; g(X)} \geq \E[f(X)]\;\E[g(X)]
\end{equation*}
Similarly, if $f$ is nonincreasing and $g$ is nondecreasing then
\begin{equation*}
    \E\sbr{f(X)\; g(X)} \leq \E[f(X)]\; \E[g(X)]
\end{equation*}
\label{thm:harris-inequality}
\end{theorem}

\begin{remark}[Mixed-monotonicity Harris inequality]
\label{remark:mixed-monotonicity-harris}
If $f$ and $g$ are both increasing or both decreasing in each coordinate, then the results of \Cref{thm:harris-inequality} still hold since we can simply negate some coordinates in the parameterization, i.e. we view $f$ and $g$ as increasing function of $-x_i$ instead of decreasing function of $x_i$. We refer to this in the proof as the mixed-monotonicity Harris inequality to highlight this subtle modification.
\end{remark}

\subsection*{Tails of sub-Gaussian distribution}

%\ascomment{needs to be Lemma E.5 for compatibility with the submitted body.}

\begin{lemma}%[Tails of sub-Gaussian distribution]
If random variable $X$ is $O(1)$-sub-Gaussian and event $E$ has probability $\Pr[E] \leq p$, then $\E[\abs{X \cdot 1_E}] \leq O(p\sqrt{\log(1/p)})$
\label{lem:sub-Gaussian-tail-bound}
\end{lemma} 
\section{Initial exploration: reduction to Incentivized RL (proof of \Cref{thm:our-HH}) }
\label{appendix:hidden-hallucination}

\OMIT{
In each round in phase $\ell$, the algorithm reveals rewards from all past-phase hallucination rounds $\calK_\ell = \{ \roundhall_{\ell'}: \ell' \in [\ell-1] \}$. A round is called past-phase if it has occurred in one of the previous phase, i.e. the same action has been taken in a previous phase. To create incentives for exploration, the reward information of these rounds is modified. If an atom has been explored at least $n_{lrn}$ times in past-phase hallucination rounds $\calK_\ell$, it is called \emph{fully-explored}. Otherwise, the atom is called \emph{under-explored}. Reward information for under-explored atoms is \emph{censored}: not revealed to the agents. Reward information for fully-explored atoms is either revealed truthfully in exploitation rounds, or hallucinated in hallucination rounds.

To formalize the reward information, we define the notion of \emph{ledger}: sequence of (action, reward) tuples from all past-phase hallucination rounds $\calK_\ell$. Depends on the type of ledger, the reward information may be modified. We consider the following four types of ledgers:
\begin{itemize}
    \item raw ledger $\ledrawl$: retains all reward information
    \item censored ledger $\ledcensl$: removes all reward information and only keep the sequence of actions.
    \item honest ledger $\ledhonl$: only retains reward for fully-explored atoms.
    \item hallucinated ledger $\ledhall$: removes all reward information from under-explored atoms, and hallucinates reward for fully-explored atoms.
\end{itemize}
The final piece needed to complete our setup is how to generate hallucinated rewards. We define the \emph{punish event} $\EvPunl$ that the reward for atom $r(a_i) \leq \epspun$ for all fully-explored atoms at phase $\ell$, where $\epspun \geq 0$ is a small parameter. In intuition, we punish the fully explored atoms by pretending that they yield small expected reward to incentivize the agents to explore other atoms. Formally, we consider the posterior distribution of the true mean vector $\theta$ given $\EvPunl$ and the censored ledger $\ledcensl$ (sequence of actions from past-phase hallucination rounds). From this distribution, we draw a hallucinated mean vector $\thetahall$. Then each time a fully-explored atom appears in an action in the ledger, we draw its reward according to this hallucinated mean $\thetahall$.

In \Cref{alg:HH}, we hide a hallucination round among multiple exploitation rounds. When an agent is shown a hallucinated ledger, they would believe that it is most likely an honest ledger. Since the reward for fully-explored atoms are small in the hallucinated ledger, the agent would choose an action containing an under-explored atom instead.

\begin{algorithm}[H]
    \caption{Hidden Hallucination}
    \label{alg:HH}
    \KwIn{Phase length $\nphase$, prior $\calP$, hallucination subroutine $\Hallucinate(\cdot)$}
    % Do something;\\
    \For{phase $\ell = 1,2,\dots$}
    {
	    Define $\phase_{\ell} = \{(\ell-1)\nphase +1,(\ell-1)\cdot \nphase +2,\dots,\ell\nphase\}$;\\
	    Draw ``hallucination round" $\kexpl$ uniformly from $\phase_{\ell}$;\\
	    Form ``honest ledger'' $\ledhonl = ((A_{\kexpi},\vR_{\kexpi})_{i = 1}^{\ell-1})$ of (actions, rewards) from past hallucination rounds;\\
	    \For{round $k \in \phase_{\ell}$}
	    {
	      \uIf{$k = \kexpl$}
	      {
	      	Reveal hallucinated ledger $\ledhall = \Hallucinate(\ledhonl,\calP)$
	      }
	      \Else
	      {
	      Reveal honest ledger $\ledhonl$.
	      }
	    }
    }
    % \Return Stuff;
\end{algorithm}
\begin{algorithm}[H]
    \caption{$\Hallucinate(\cdot)$ subroutine}
    \label{alg:HH_subroutine}
    \KwIn{Honest ledger $\ledhonl$, prior $\calP$.}
    Form censored ledger $\ledcensl = (A_{k_{i}})_{i=1}^{\ell-1}$, which is just the actions in the honest ledger; \\
    \algcomment{\% $N_{\ell}(a)$ denotes the \# of actions $A_{k_i}, i < \ell$ containing coordinate $a$}\\
    Draw $\thetahall \sim \Pr_{\theta \sim \calP}[ \cdot \mid \forall a: N_{\ell}(a) \ge \nlearn, \mu_a \le \epspun ]$;
    \\
    \For{$i=1,2\dots,\ell-1$}
    {
    Set $\tilde{\vR}_{i} = \thetahall (A_{k_i}) + \mathcal{N}(0,\eye_{A_{k_i}})$
    }
    \Return hallucinated ledger $\ledhall = ((A_{k_{i}},\tilde{\vR}_{i})_{i=1}^{\ell-1})$ ;	
\end{algorithm}

The guarantees of Hidden Hallucination mechanism is originally stated using RL notations. In order to derive the guarantees for combinatorial semi-bandits setting, we first provide a lemma stating the direct reduction from episodic RL to combinatorial semi-bandits.

\begin{lemma}[Reduction to combinatorial semi-bandits]
\label{lem:reduction-RL}
Let $M$ be an MDP model with $S$ states, $A$ actions, and $H$ stages, where $H$ is the time horizon. Then in the special case where the time horizon $H = d$, the action at each stage is either picking an atom or not, and number of states $S = d$, then we recover a combinatorial semi-bandit problem.
\end{lemma}
\begin{proof}
\dndelete{Let the time horizon $H = d$, number of states $S=d$ and number of actions $A = 2$, then each atom of an action $A_i$ in combinatorial semi-bandit can be indexed by a state-action pair $(s, a) $. Each combinatorial action $A_i$ corresponds to a policy that either picks an atom or not at each state. }
In the episodic RL setting, we use a standard model of Markov Decision Process (MDP) without time discounting. We have the following finite parameters: $S$ states, $A$ actions and $H$ stages, where $H$ is the time horizon. We write $x \in [S]$, $a \in [A]$ and $h \in [H]$ to represent states, action and stages, respectively. In our notation, we often refer to $(x,h,a)$ triples.

There is a set of feasible $(x,a,h)$ triples called $\mathrm{FEASIBLE}$, which is a subset of $S \times A \times H$. We modify the original MDP setting by including two terminal states, $\mathrm{GOOD}$ and $\mathrm{BAD}$. We deterministically go to $\mathrm{BAD}$ state if a triple $(x,a,h) \notin \mathrm{FEASIBLE}$. Otherwise, at the end of the algorithm we go to $\mathrm{GOOD}$.

In the prior, we know deterministically that $\mathrm{BAD}$ yields reward $0$, while $\mathrm{GOOD}$ yield reward $H+1$. That is, the reward of $\mathrm{GOOD}$ is always high enough to ensure agents would always choose feasible policies, i.e. policies that only contain $(x,a,h) \in \mathrm{FEASIBLE}$.

Each RL action corresponds to either picking atom or not at each state. Number of states $A$ corresponds to number of actions in combinatorial bandit. The MDP transitions are deterministic and the transition graph is a directed acyclic graph. By adhering to $\mathrm{FEASIBLE}$, the agents cannot choose the same action twice in the same MDP. The policies are Markovian and deterministic. Each policy uniquely specifies a sequence of RL actions, which corresponds to uniquely choosing a subset of atoms. The reward of a given RL action is the reward of the corresponding atom.

In our setting where all combinatorial bandit actions are $m$-subset of atoms, we can encode them as follow. Let the time horizon $H = d$ and index the atoms as $1, 2, \cdots, d$. Each state in a MDP is the largest atom previously chosen. Then, the feasible actions are all larger atoms.

\end{proof}
}

This appendix spells out the analysis for \Cref{subsec:hh}: the approach for initial exploration by reduction to incentivized reinforcement learning (RL) \citep{IncentivizedRL}.

We build on an algorithm from \citet{IncentivizedRL}, called \algHH, and their guarantee for this algorithm. We state their setup and guarantee below. (The specification of their algorithm is unimportant for our presentation.) Then we use it to prove \Cref{thm:our-HH}.

\subsection*{Incentivized RL: the setup and the guarantee}

The setting is as follows. Consider an MDP with $S$ states, $A$ actions and $H$ stages, where $H$ is the time horizon. We write $x \in [S]$, $a \in [A]$, and $h \in [H]$ to represent states, actions, and stages, respectively. In the following analysis, we often refer to $(x,a,h)$ triples. We consider a set of feasible $(x,a,h)$ triples called $\feaSet \subset [S]\times[A]\times[H]$. (In \citet{IncentivizedRL}, $\feaSet = [S]\times[A]\times[H]$, but we will extend their result to an arbitrary $\feaSet$.)

The ``true" MDP is denoted by $\phi$. Let $r_{\phi}(x,a,h)$ be the expected reward if action $a$ is chosen at state $x$ and stage $h$. We posit a Bayesian model: $\phi$ is chosen from a Bayesian prior $\calP$.

Then, we consider the setting of \emph{episodic RL}, where in each episode $t$ an algorithm chooses a policy $\pi\supt$ in this MDP.%
\footnote{By default, MDP policies are Markovian and deterministic.}
The chosen policies must satisfy a similar BIC condition: for each round $t\in[T]$,
\begin{align*}%\label{def:BIC-RL}
    \E[\; V(\pi) - \mu(\pi') \mid \pi\supt = \pi\;] \geq 0
    \qquad\forall \text{ policies } \pi,\pi'\in\calA \text{ with } \Pr[\pi\supt=\pi]>0,
\end{align*}
where $V(\pi)$ is the value (expected reward) of policy $\pi$. Essentially, this is the same condition as \eqref{def:BIC}, where arms are replaced with policies.

We only need the guarantee for \algHH for an MDP with deterministic transitions (but randomized rewards). This guarantee depends on the following prior-dependent quantities:
\begin{align}
    \qpun(\eps) &\coloneqq \Pr[r_{\phi}(x,a,h) \leq \eps, \forall (x,a,h) \in \feaSet], \label{eq:q_pun}\\
    \ralt &\coloneqq \min_{(x,a,h) \in \feaSet} \E[r_{\phi}(x,a,h)]. \label{eq:r_alt}
\end{align}
The guarantee is stated as follows.

\begin{theorem}
Consider an arbitrary prior $\calP$. Fix parameters $\delta \in (0,1]$. Assume that $\ralt > 0$ and $\qpun = \qpun(\epspun) > 0$, where $\epspun = \mathrm{r_{alt}}/18H$. Consider \algHH with punishment parameter $\epspun$, appropriately chosen phase length $\mathrm{n_{ph}}$, and large enough target $n = \mathrm{\nlearn}$. This algorithm is guaranteed to explore all $(x,a,h) \in \feaSet$ with probability at least $1 - \delta$ by round $N_0$, where $n$ and $N_0$ are specified below.

For some absolute constants $c_1, c_2$, it suffices to take \begin{align*}
    n &= \nlearn \geq c_1 \cdot \ralt^{-2} H^4 \rbr{S + \log \frac{SAH}{\delta \cdot \ralt \cdot \qpun}},\\
    N_0 &= c_2 \cdot n \cdot \qpun \cdot \ralt^{-3} \cdot SAH^4
\end{align*}
In particular, for any $n \geq 1$, one can obtain $N_0 \leq n \cdot \qpun \cdot \poly \rbr{\ralt^{-1} SAH } \cdot \log \rbr{\delta^{-1} \qpun^{-1}}$.
\label{thm:HH-randomized-reward}
\end{theorem}

As we mentioned above, \citet{IncentivizedRL} guarantees \Cref{thm:HH-randomized-reward} for $\feaSet = [S]\times[A]\times[H]$. Below, we show how to extend it to an arbitrary $\feaSet = [S]\times[A]\times[H]$.

%Our result is a slight modification of the following \algHH guarantee for arbitrary priors in \citet{IncentivizedRL}. In their setting, the \algHH guarantee holds for the case where $\feaSet = [S]\times[A]\times[H]$.

\begin{theorem}[\citet{IncentivizedRL}]
\label{thm:HH-SimSliv21}
The guarantee in \Cref{thm:HH-randomized-reward} holds for $\feaSet = [S]\times[A]\times[H]$.
\end{theorem}

\begin{remark} The relevant result, Theorem 5.5 in \citet{IncentivizedRL} is stated for MDPs with randomized transitions. In this more general formulation, \eqref{eq:q_pun} and \eqref{eq:r_alt} are conditioned on an object called \emph{censored ledger}, and then a minimum is taken over all such objects. However, this conditioning vanishes when the MDP transitions are deterministic. (This follows easily from Lemma 6.2 in \citet{IncentivizedRL}, essentially because censored ledgers do not carry any useful information.) We present a version without censored ledgers, because defining them is quite tedious.
%spelling it out requires introducing many definitions from \citet{IncentivizedRL},
%, the prior-dependent constants use the notation of censored ledger $\lambda$, which is a sequence of policy-trajectory pairs from all past-phase hallucination episodes with no reward information. That is,
%\begin{align}
%    \qpun(\eps) &\coloneqq \min_{\text{censored-ledger }\lambda} \Prcan[r_{\theta}(x,a,h) \leq \eps \leq \eps, \forall (x,a,h) \in \feaSet \mid \lambda]\\
%    \ralt &\coloneqq \min_{\text{censored-ledger }\lambda} \min_{(x,a,h) \in \feaSet} \Ecan[r_{\theta}(x,a,h) \mid \lambda]
%\end{align}
%Note that when the MDP transitions are deterministic and under ledger hygiene \citep{IncentivizedRL}, the censored ledger does not carry any useful information. Hence, for our analysis, we can simplify the two prior-dependent constants as stated in \Cref{eq:q_pun,eq:r_alt}.
\end{remark}

\begin{proof}[Proof Sketch for \Cref{thm:HH-randomized-reward}]
Start with an arbitrary $\feaSet$. We modify the MDP as follows. Add two terminal state, \goodSet and \badSet, such that where we deterministically transition into \badSet if $(x,a,h) \notin \feaSet$. Otherwise, at the end of the MDP, we go into \goodSet. We let \badSet yield reward $0$, and \goodSet yield reward $H+1$. With this modified MDP, even if all $(x,a,h)$ triples are allowed, any BIC algorithm would only choose feasible policies, \ie policies that only choose $(x,a,h) \in \feaSet$. So, we conclude by invoking \Cref{thm:HH-SimSliv21}.
%We proceed to modify the MDP as follows. The MDP now includes two terminal state, \goodSet and \badSet, where we deterministically go into the \badSet if $(x,a,h) \notin \feaSet$. Otherwise, at the end of the MDP, we go into \goodSet. We want to ensure the agents would only choose feasible policies, \ie policies that only choose $(x,a,h) \in \feaSet$. Hence, we let \badSet reward be $0$, and \goodSet reward be $H+1$. With this modified MDP and \Cref{thm:HH-SimSliv21}, we conclude the proof.
\end{proof}

\subsection*{Proof of \Cref{thm:our-HH} for combinatorial semi-bandits}

Now, let us go back to combinatorial semi-bandits and prove \Cref{thm:our-HH}. We start with an instance of combinatorial semi-bandits and construct an MDP as specified in \Cref{subsec:hh}. Then we invoke \Cref{thm:HH-randomized-reward}. To state the final guarantee, it remains to interpret (and simplify) the quantities in \Cref{thm:HH-randomized-reward} for a particular MDP obtained with our construction. 

%To derive the guarantee of \algHH for combinatorial semi-bandits, we need to write the two prior-dependent parameters $\qpun$ and $\ralt$ in bandit notation.

First, recall that $H=A=d$ and $S\leq d$, where $d$ is the number of atoms.

Second, $r_\phi(x,a,h)$ is simply $\theta_\ell$, the expected reward of the corresponding atom $\ell$.  Accordingly, 
\begin{align*}
    \ralt 
        = \min_{\text{atoms }\ell \in [d]} \E[\theta_\ell]
        \geq \min_{\text{priors }\calP_\ell \in \calC} \E_{\theta_\ell\sim\calP_\ell}[\theta_\ell]
        := \eps_0.
\end{align*}
Note that $\eps_0$ is determined by the collection $\calC$ of feasible per-atom priors. 

Finally, observe that $\qpun$ is the probability of all $(x,a,h)$ triples have low reward:
\begin{align*}
\qpun(\eps)
        &= \Pr[\theta_\ell \leq \eps:\; \forall \text{ atoms }\ell \in [d]]
\end{align*}
We can divide all $(x,h,a)$ triples into classes, where each class represents an atom. Since our prior is independent across the atoms,
\begin{align*}
\qpun &= \qpun(\ralt) \geq \qpun(\eps_0) \\ 
    &= \Pr\sbr{ \theta_\ell \leq \eps_0:\; \forall \text{ atoms }\ell \in [d]}\\
    &= \prod_{\ell\in[d]} \Pr\sbr{\theta_\ell \leq \eps_0}\\
    &\geq \rbr{ \min_{\text{priors }\calP_\ell \in \calC} 
         \Pr_{\theta_\ell\sim\calP_\ell}\sbr{\theta_\ell \leq \eps_0} }^d.
\end{align*}
Again, the expression in $(\cdot)$ is determined by the collection $\calC$.

%Finally, we observe that it suffices to take  $S,A,H$ in the MDP to be $d$. Then, using \Cref{thm:HH-randomized-reward} and the substitutions above, we can derive the guarantee of \Cref{thm:our-HH}. 
\end{document}